\documentclass{article}

\usepackage{microtype}
\usepackage{graphicx}
\usepackage{subcaption}
\usepackage{booktabs} %

\usepackage{hyperref}
\hypersetup{hypertexnames=false}
\usepackage{silence}
\usepackage[accepted]{icml2026}

\usepackage{amsmath}
\usepackage{amssymb}
\usepackage{mathtools}
\usepackage{amsthm}

\usepackage{algorithm}

\usepackage{algpseudocode}
\makeatletter
\renewcommand*{\theHALG@line}{\thealgorithm.\arabic{ALG@line}}
\makeatother

\usepackage[capitalize,noabbrev]{cleveref}

\theoremstyle{plain}
\newtheorem{theorem}{Theorem}
\newtheorem{proposition}{Proposition}
\newtheorem{lemma}{Lemma}

\theoremstyle{definition}
\newtheorem{definition}{Definition}
\newtheorem{assumption}{Assumption}
\theoremstyle{remark}
\newtheorem{remark}{Remark}

\newcommand{\cA}{\mathcal{A}}

\newcommand{\cP}{\mathcal{P}}

\newcommand{\cS}{\mathcal{S}}

\renewcommand{\H}{\mathsf{H}}
\newcommand{\E}{\mathbb{E}}

\newcommand{\wass}{\mathsf{W}}

\newcommand{\KL}{\mathsf{KL}}

\providecommand{\argmax}{\mathrm{argmax}}

\renewcommand{\d}{\mathrm{d}}

\newcommand{\R}{\mathbb{R}}

\newcommand{\Prob}{\mathbb{P}}

\newcommand{\W}{W_2}
\newcommand{\ip}[2]{\left\langle #1,#2\right\rangle}
\newcommand{\osc}{\operatorname{Osc}}

\newcommand{\Law}{\mathop{\mathrm{Law}}}
\usepackage[textsize=tiny]{todonotes}

\usepackage{enumitem}

\icmltitlerunning{Wasserstein Proximal Policy Gradient}
\begin{document}

\twocolumn[
  \icmltitle{Wasserstein Proximal Policy Gradient}

  \begin{icmlauthorlist}
    \icmlauthor{Zhaoyu Zhu}{yyy}
    \icmlauthor{Shuhan Zhang}{comp}
    \icmlauthor{Rui Gao}{sch}
    \icmlauthor{Shuang Li}{comp}
  \end{icmlauthorlist}

  \icmlaffiliation{yyy}{Zhiyuan College, Shanghai Jiao Tong University, Shanghai 200240, China}
  \icmlaffiliation{comp}{School of Data Science, The Chinese University of Hong Kong, Shenzhen, Guangdong, China}
  \icmlaffiliation{sch}{McCombs School of Business, The
University of Texas at Austin, Austin, TX, USA}

  \icmlcorrespondingauthor{Shuang Li}{lishuang@cuhk.edu.cn}

  \icmlkeywords{Reinforcement learning, Continuous control, Global convergence, Optimization on the probability space}

  \vskip 0.3in
]

\printAffiliationsAndNotice{ }  %

\begin{abstract}
We study policy gradient methods for continuous-action, entropy-regularized reinforcement learning through the lens of Wasserstein geometry. Starting from a Wasserstein proximal update, we derive Wasserstein Proximal Policy Gradient (WPPG) via an operator-splitting scheme that alternates an optimal transport update with a heat step implemented by Gaussian convolution. This formulation avoids evaluating the policy's log density or its gradient, making the method directly applicable to expressive implicit stochastic policies specified as pushforward maps. We establish a global linear convergence rate for WPPG, covering both exact policy evaluation and actor–critic implementations with controlled approximation error. Empirically, WPPG is simple to implement and attains competitive performance on standard continuous-control benchmarks.

\end{abstract}

\section{Introduction}

Reinforcement learning (RL) has become a powerful paradigm for solving complex sequential decision-making problems, powering landmark achievements from superhuman performance in strategic games\citep{silver2016mastering,silver2017mastering} and advanced robotic control \citep{levine2016end} to the training of Large Language Models \citep{guo2025deepseek}. At the heart of many of these successes are policy gradient (PG) methods \citep{williams1992simple, sutton1999policy}, which iteratively update a parameterized policy to maximize expected rewards.

The geometry underlying policy updates plays an important role in the learning process. Standard policy gradient methods use the Euclidean geometry of parameter space, while the natural policy gradient \citep{kakade2001natural} and trust-region methods such as TRPO \citep{schulman2015trust} and PPO \citep{schulman_proximal_2017} instead exploit the information geometry of policies via Kullback–Leibler (KL) divergence.
These methods are supported by a growing body of analysis, with recent results establishing fast global convergence rates in finite action spaces \citep{agarwal2021theory,lan2023policy,xiao2022convergence,cen2022fast,bhandari2024global}.

Recent work explores an alternative paradigm that formulates policy optimization in distribution space under the Wasserstein metric.
This perspective builds on the theory of gradient flows in probability spaces \citep{zhang2018policy,moskovitz2021efficient,ziesche2023wasserstein,pfau_wasserstein_2025}.
In contrast to KL-based methods, which treat actions as independent categories, Wasserstein-based approaches inherently respect the geometry of the action space, capturing meaningful notions of proximity between actions \citep{pacchiano2020learning, moskovitz2021efficient, song_provably_2023}. The resulting stochastic policy updates rely on the gradient of the action-value function with respect to the action, drawing a close connection to deterministic policy gradients \citep{pfau_wasserstein_2025}.

While Wasserstein policy optimization offers a compelling alternative to KL-based methods, its theoretical foundations are far less developed.
\citep{terpin2022trust} uses Wasserstein distance as the trust region metric, but they explicitly leave the convergence analysis for future work. Using the JKO framework \citep{jordan1998variational} for Wasserstein gradient flows, \citet{zhang2018policy} established asymptotic convergence of the entropy-regularized problem when the policy distribution is approximated by particles.
For finite action spaces, \citet{song_provably_2023} proved global convergence as the Wasserstein penalty coefficient vanishes.
Beyond these results, however, convergence guarantees in the more general setting of continuous action spaces---particularly for parametric policies beyond particle approximations (e.g., mixtures of Gaussians)---remain, to the best of our knowledge, an open question.

In this paper, we introduce a version of the Wasserstein policy gradient, which we term Wasserstein Proximal Policy Gradient (WPPG).
Our main findings are as follows.

\begin{itemize}[leftmargin=2em]
\item Our WPPG update introduces a new scheme for optimizing stochastic policies. Via an operator-splitting, we decompose the Wasserstein proximal policy update into two steps: a Wasserstein transport step that shifts actions to increase the action-value function, followed by a heat flow step that injects Gaussian noise to account for entropy regularization. When instantiated with parametric policies, it does not require access to the policy distribution's (log-)density or its gradient.
This enables a novel approach to policy optimization with \emph{implicit policies} \citep{tang2018implicit}. Empirically, the resulting algorithm is simple to implement and demonstrates competitive performance on standard continuous-control benchmarks.

\item We establish a linear convergence rate of WPPG for the entropy-regularized problem, under assumptions that can be verified directly. Our analysis applies to both the exact and approximate value function estimation.
\end{itemize}

\subsection{Related Works}

\paragraph{On Wasserstein policy update}
\citet{zhang2018policy} formulate continuous-time entropy-regularized policy optimization as a gradient flow of an energy functional, and derive its discrete-time counterpart via the JKO scheme \citep{jordan1998variational}, parameterizing policies with particles or energy-based models.
Related gradient-flow perspectives also appear in \citet{richemond2018diffusing}, and in \citet{ziesche2023wasserstein} for mixtures of Gaussian policies.
\citet{moskovitz2021efficient} introduce a Wasserstein trust-region policy update and develop efficient kernel-based estimators.
\citet{song_provably_2023} study Wasserstein and Sinkhorn trust-region updates in finite action spaces, with an extension to one-dimensional continuous control, and derive closed-form policy updates via duality.
More recently, \citet{pfau_wasserstein_2025} propose a Wasserstein gradient flow–inspired update by projecting the Wasserstein gradient flow on parametric manifold using KL divergence.

Most of the above-mentioned Wasserstein policy updates rely on the (log-)density of the policy distribution (or probability mass function for finite action spaces) and/or its score functions, with the exception of \citet{moskovitz2021efficient}, whose kernel-based method instead requires gradients of the kernel for implicit models.
In particular, our update, like those of \citet{zhang2018policy,pfau_wasserstein_2025}, depends on the gradient of the action-value function with respect to the action. However, unlike both of these two approaches, our method does not rely on the (log-)density of the policy distribution. This is achieved by handling the entropy term through Gaussian noise injection, rather than working directly with the density as in \citet{zhang2018policy}, and by projecting under the Wasserstein metric instead of the KL divergence, as in \citet{pfau_wasserstein_2025}.

We remark that other related approaches leverage Wasserstein geometry in different ways. For instance, \citet{pacchiano2020learning} compare policies in latent behavior spaces but do not focus on explicit policy updates, while \citet{abdullah2019wasserstein} use Wasserstein distances for robust uncertainty modeling, treating Wasserstein as a constraint on the transition dynamics of the environment rather than as a tool for defining learning dynamics.

\paragraph{On convergence analysis}
\citet{zhang2018policy} establish asymptotic convergence of Wasserstein policy optimization based on the JKO scheme when policies are approximated by particles.
\citet{song_provably_2023} proves linear convergence analysis for Wasserstein trust-region optimization on finite action spaces, but requires vanishing Wasserstein penalty coefficients.
Our convergence proof strategy parallels KL-based analyses such as \citet{lan2023policy} for mirror policy gradient in finite action spaces. However, instead of relying on KL-specific tools (e.g., the three-point identity), we develop new analyses tailored to Wasserstein geometry and adopt different assumptions on the problem.

\section{Preliminaries}

\subsection{Reinforcement Learning}
\paragraph{Markov decision processes.}
We consider an infinite-horizon discounted MDP
\(
\mathcal M=(\mathcal S,\mathcal A,P,r,\rho,\gamma)
\),
where \(\mathcal S\) is the state space, \(\mathcal A\) is the action space,
\(\gamma\in(0,1)\) the discount factor, \(\Prob(\cdot\mid s,a)\) the transition kernel,
\(r:\mathcal S\times\mathcal A\to\mathbb R\) the reward function,
and \(\rho\) the initial-state distribution.
A policy is a Markov kernel
\(\pi:\mathcal S\to \cP (\mathcal A)\) so that \(\pi(\cdot\mid s)\) is a probability
distribution on \(\mathcal A\) for each \(s\in\mathcal S\).
Our results apply to the case where $\cA$ is a general metric space.

For a policy \(\pi\), the value and action-value (Q) functions are
\[
\begin{aligned}
V^\pi(s)
&:=\mathop{\mathbb{E}}\limits_{\substack{a_t\sim \pi(\cdot\mid s_t)\\
s_{t+1}\sim \Prob(\cdot\mid s_t,a_t)}}\left[\sum_{t=0}^{\infty}\gamma^tr(s_t,a_t)\bigg|s_0=s\right],\\
Q^\pi(s,a)
&:= \mathop{\mathbb{E}}\limits_{\substack{a_t\sim \pi(\cdot\mid s_t)\\
s_{t+1}\sim \Prob(\cdot\mid s_t,a_t)}}\left[\sum_{t=0}^{\infty}\gamma^t\,r(s_t,a_t)\,\bigg|\,s_0=s,\ a_0=a\right].    
\end{aligned}
\]
We have the relationship that \[
V^\pi(s)= \E_{a\sim\pi(\cdot|s)} \left[Q^\pi(a,s)\right].
\]
The advantage function is \(A^\pi(s,a):=Q^\pi(s,a)-V^\pi(s)\).
Given an initial distribution \(\rho\), the performance (expected return) is
\(
J_\rho(\pi):=\mathbb E_{s_0\sim \rho}[\,V^\pi(s_0)\,]
\). The discounted state visitation probability is defined as $d_\rho^\pi(s)
:=(1-\gamma)\sum_{t=0}^\infty \gamma^t\,\mathbb P_\pi(s_t=s\mid s_0\sim\rho)$, with $\sum_{s} d_\rho^\pi(s)=1$.

\paragraph{Explicit vs.\ implicit policies.}
We distinguish explicit and implicit policies by whether the log-density is available.
An explicit policy is one for which the policy's log-density \(\log \pi_\theta(a\mid s)\) can be evaluated directly for any state-action pair \((s,a)\).
In contrast, an implicit policy is specified only through a transport map,
\[
a = g_\theta(s,Z), \qquad Z \sim \nu,
\]
where \(\nu\) is an easy-to-sample distribution.
The induced policy \(\pi_\theta(\cdot\mid s) = (g_\theta(s,\cdot))_\# \nu\) generally has an unknown density, so \(\log \pi_\theta(a\mid s)\) is unavailable and its density is hard to compute.
A subtle but important distinction is that a pushforward representation does not automatically imply implicitness in our sense: for example, a Gaussian policy can be written as the transport
\(g_\theta(s,Z) = \mu_\theta(s) + \sigma_\theta(s) Z\),
yet it remains an \emph{explicit} policy because its log-density admits a closed-form expression.
Explicit policies are often more expressive than common explicit families (e.g., Gaussians), since a rich generator $g_\theta$
can represent complex and non-Gaussian(Multimodal) action distributions while remaining easy to sample from, which is particularly attractive in continuous-control tasks.

\subsection{Entropy Regularization}\label{sec:entropy}
Entropy regularization is widely used in reinforcement learning to prevent premature policy collapse and encourage exploration. It smooths the optimization landscape, reduces gradient variance, and promotes robust, generalizable strategies. These benefits make it a standard component in modern algorithms like soft Q-learning and Soft Actor-Critic (SAC).
Define the negative entropy of a policy $\pi$ at state $s$ as 
\[
\H^\pi(s) := \int_{a\in\cA}\pi(a|s)\log\pi(a|s)\d a.
\]
The entropy-regularized  discounted value is the value of \(\pi\) under the modified reward \(r_\tau(s,a):=r(s,a)-\tau\log\pi(a\mid s)\):
\[
V_\tau^\pi(s) := \E^\pi \Big[\sum_{t=0}^{\infty}\gamma^t\big(r(s_t,a_t)-\tau \H^\pi(s_t)\big) \Big| s_0=s\Big],
\]
Define the corresponding soft-\(Q\) function by
\[
Q_\tau^\pi(s,a)
:=
r(s,a) + \gamma \E\left[V_\tau^\pi(s')\mid s,a\right].
\]
Then the soft Bellman recursion can be written as
\begin{equation}\label{eq:soft_value}
\begin{aligned}
V_\tau^\pi(s)
&=
\E_{a\sim \pi(\cdot\mid s)}\left[Q_\tau^\pi(s,a) - \tau \log \pi(a\mid s)\right]
\\
&=
\E_{a\sim \pi(\cdot\mid s)}\left[Q_\tau^\pi(s,a)\right] - \tau \H^\pi(s).
\end{aligned}
\end{equation}

\paragraph{Connection to injected Gaussian noise}
In distributional optimization, entropy regularization corresponds to injecting Gaussian noise in Langevin dynamics. More precisely, consider a generic distributional optimization problem over probability measures \(\mu\in\mathcal P(\R^d)\):
\begin{equation}\label{eq:free_energy_generic}
\min_{\mu\in\mathcal P(\R^d)}
\ \mathcal F(\mu)
:= \langle U, \mu \rangle + \tau \H(\mu).
\end{equation}
A standard way to solve it is the Langevin dynamics:
\begin{equation}\label{eq:langevin_generic}
x_{k+1}
=
x_k - \eta \nabla U(x_k) + \sqrt{2\eta\tau} \xi_k,
\quad
\xi_k\sim\mathcal N(0,I).
\end{equation}
Hence, the entropy penalty \(\tau\) appears exactly as the variance scale of the injected Gaussian noise.

\paragraph{Policy-entropy vs.~parameter-entropy.}
It is important to distinguish the above \emph{policy entropy} regularization from approaches that inject Gaussian noise at the level of policy parameters.
For example, \citet{optrl2019_langevin} models a distribution over policy parameters and performs Langevin-type updates on the parameters.
While this also promotes exploration, it does not directly regularize the conditional action distribution \(\pi(\cdot\mid s)\), and therefore does not yield the same soft Bellman structure in~\eqref{eq:soft_value}.
In particular, parameter-space randomness can induce state-dependent action stochasticity only indirectly through the parameter-to-action map, whereas policy-entropy regularization explicitly regularizes uncertainty at the level of the action distribution.

\subsection{Wasserstein Policy Gradient}
\label{sec: Wasserstein Proximal Gradient}

We consider a parameterized stochastic policy $\pi_\theta(\cdot\mid s)$ on a continuous action space.
Let $J(\pi_\theta)$ denote the expected discounted return obtained by executing $\pi_\theta$.
Policy gradient methods update $\theta$ so as to increase $J(\pi_\theta)$ by following an estimator of its gradient.
A central identity is the policy gradient theorem, which expresses the gradient in terms of the action-value function:
\[
\nabla_\theta J(\pi_\theta)
=
\mathbb{E}_{s\sim d^{\pi_\theta},\,a\sim \pi_\theta(\cdot\mid s)}
\Big[\nabla_\theta \log \pi_\theta(a\mid s) Q^{\pi_\theta}(s,a)\Big].
\]
Thanks to this result, most policy gradient methods work with \emph{explicit} policies.

Wasserstein policy optimization operates directly in the space of action distributions $\pi(\cdot\mid s)$ equipped with the 2-Wasserstein metric.
This viewpoint controls how the policy changes in action space and connects policy improvement to Wasserstein gradient flows and optimal transport geometry.
Fix a state $s$ and iteration $k$.
The Wasserstein proximal policy update is given by
\begin{equation}
\label{eq:jko_policy_short}
\begin{aligned}
\pi_{k+1}&(\cdot\mid s)
\in
\argmax_{\pi(\cdot\mid s)}\\
&
\langle Q^{\pi_k}(s,\cdot),\,\pi(\cdot\mid s)\rangle
-\frac{1}{2\eta}\mathsf{W}_2^2\big(\pi(\cdot\mid s), \pi_k(\cdot\mid s)\big),
\end{aligned}
\end{equation}
where $\mathsf{W}_2$ is 2-Wasserstein metric, $\eta>0$ is a step size and $\langle f,\mu\rangle:=\int f\,\mathrm d\mu$.
Equivalently, $Q^{\pi_k}(s,\cdot)$ can be replaced by the advantage function
$A^{\pi_k}(s,\cdot)=Q^{\pi_k}(s,\cdot)-V^{\pi_k}(s)$, since $V^{\pi_k}(s)$ is constant with respect to the optimized action distribution.
The Wasserstein penalty in \eqref{eq:jko_policy_short} plays the role of a trust region, encouraging local improvement while preventing large shifts of the policy in action space. Alternatively, one can work with an explicit trust-region constraint on the Wasserstein distance.

A complementary implementation strategy is from the perspective of Wasserstein gradient flow, which realizes the Wasserstein update through the action-gradient of the critic.
Concretely, actions are transported along the drift induced by $\nabla_a Q^{\pi_k}(s,a)$:
\[
a \ \mapsto\ a + \eta\,\nabla_a Q^{\pi_k}(s,a),
\qquad a\sim \pi_k(\cdot\mid s),
\]
and the updated policy is obtained by pushing forward $\pi_k(\cdot\mid s)$ through this map.

To update the policy parameters $\theta$, existing literature works within the tabular or explicit policy setting, relying on access to the policy density or score function. Different from these works, our approach in the next section develops a method aimed at implicit policies that bypasses the need for density or score function evaluations.

\section{Wasserstein Proximal Policy Gradient for Entropy-Regularized RL}\label{sec:wppg}

\subsection{Derivation of the Policy Update}

For an initial-state distribution, we evaluate a policy $\pi$ by the expected performance
\[
J_\rho(\pi)
:=\E_{s_0\sim\rho}\big[V_\tau^\pi(s_0)\big].
\]
Substituting the soft Bellman identity \eqref{eq:soft_value} into $J_\rho(\pi)$ yields a composite objective consisting of an expectation term and an entropy term. We propose to optimize $J_\rho(\pi)$ via a proximal gradient scheme in the Wasserstein space of policies.

To this end, we first note that the first variation of $J_\rho(\pi)$ with respect to the conditional distribution $\pi(\cdot\mid s)$ can be expressed as
\[
\frac{\delta}{\delta \pi(\cdot\mid s)} J_\rho(\pi)(a) = \frac{1}{1-\gamma} d_\rho^{\pi}(s) Q_\tau^{\pi}(s,a),
\]
where $d_\rho^{\pi}(s)$ is the discounted visitation distribution under $\pi$, up to an additive constant in $a$.
Thereby, at iteration $k$, the linearization of $J_\rho(\pi)$ around the current policy $\pi_k$ is given by
\[
\frac{1}{1-\gamma}\E_{s\sim d_\rho^{\pi_k}} \E_{a\sim \pi(\cdot\mid s)} [Q_\tau^{\pi_k}(s,\cdot)].
\]
Consider a weighted Wasserstein distance that aggregates across states with weights $d_\rho^{\pi_k}(s)/(1-\gamma)$. Then the resulting Wasserstein proximal update is given by
\begin{equation}\label{eq:WPPG}
\begin{aligned}
& \pi_{k+1}(\cdot\mid\cdot)
\in
\argmax_{\pi} \Big\{ \E_{s\sim d_\rho^{\pi_k}} \E_{a\sim \pi(\cdot\mid s)} [Q_\tau^{\pi_k}(s,\cdot) ]
\\
& - \frac{1}{2\eta}
\E_{s\sim d_\rho^{\pi_k}}
\big[
\mathsf{W}_2^2\big(\pi(\cdot\mid s),\pi_k(\cdot\mid s)\big)
\big] - \tau \E_{s\sim d_\rho^{\pi_k}}[\H^\pi(s)] \Big\}.
\end{aligned}
\tag{\texttt{WPPG}}
\end{equation}
Because it decomposes across states, this is equivalent to 
\begin{equation}
\begin{aligned}
\pi_{k+1}(\cdot\mid s) 
\in &
~\argmax_{\pi(\cdot\mid s)} 
\Big\{
\langle Q_\tau^{\pi_k}(s,\cdot), \pi(\cdot\mid s) \rangle \\
& -\frac{1}{2\eta}
\wass_2^2\big(\pi(\cdot\mid s), \pi_k(\cdot\mid s)\big) -\tau \H^q(s)
\Big\},
\end{aligned}
\tag{\texttt{WPPG}$_s$}\label{eq:WPPGs}
\end{equation}
which is termed the Wasserstein proximal policy gradient (WPPG).
When \(\tau=0\), it reduces to the entropy-free Wasserstein proximal update \eqref{eq:jko_policy_short}.
When $\mathsf{W}_2^2$ is replaced by a Bregman divergence, \eqref{eq:WPPGs} recovers mirror-descent style policy optimization \cite{lan2023policy}.
Note that $Q_\tau^{\pi_k}(s,a)$ can replaced by the advantage function $A_\tau^{\pi_k}(s,a)=Q_\tau^{\pi_k}(s,a)-V_\tau^{\pi_k}(s)$, since $V_\tau^{\pi_k}(s)$ does not depend on $\pi(\cdot\mid s)$.
We would like to emphasize that this per-state formulation will be mainly used for theoretical analysis; below, we will present a practical scheme that works with \emph{implicit policies}.

To solve \eqref{eq:WPPG} (or \eqref{eq:WPPGs}), we employ a (Lie–Trotter) operator splitting that decouples the Wasserstein proximal component from the entropy component by maintaining a full-step sequence $\{\pi_k\}$ and a half-step sequence $\{\pi_{k-\frac12}\}$. 
We perform the following two steps sequentially.

First, we compute an intermediate policy $\pi_{k+\frac12}(\cdot\mid s)$ by
\begin{equation}\label{eq:split_transport}
\begin{aligned}
&\pi_{k+\frac12} (\cdot\mid s)
\in
\argmax_{\pi(\cdot\mid s)}\\
& \quad \langle Q_\tau^{\pi_{k-\frac12}}(s,\cdot), \pi(\cdot\mid s)\rangle
-\frac{1}{2\eta}\mathsf{W}_2^2\big(\pi(\cdot\mid s),\pi_{k-\frac12}(\cdot\mid s)\big),
\end{aligned}
\end{equation}
Here, the half sequence can be initialized at any implicit policy $\pi_{\frac12} = g_{0\#}\nu$.
Second, to handle the entropy term, using its connection to injected Gaussian noise (Section \ref{sec:entropy}),
we convolve $\pi_{k+\frac12}(\cdot\mid s)$ with a Gaussian kernel:
\begin{equation}
\pi_{k+1}(\cdot\mid s)
= \pi_{k+\frac12}(\cdot\mid s)*\mathcal N(0,2\tau\eta I\big).
\label{eq:split_heat}
\end{equation}
Equivalently, if $A \sim \pi_{k+\frac12}(\cdot\mid s)$ and $\xi\sim\mathcal N(0,I)$ are independent, then
$A + \sqrt{2\tau\eta} \xi \sim \pi_{k+1}(\cdot\mid s)$.

\subsection{Wasserstein Proximal Gradient for Implicit Policies}
We now discuss how to solve \eqref{eq:split_transport} for implicit polices.
Recall that an implicit policy at state \(s\) is represented as a pushforward \(\pi(\cdot\mid s)=g(s,\cdot)_\#\nu\), where $g$ is a sufficiently expressive generative model.
We represent the solution to \eqref{eq:split_transport} as $\pi_{k+\frac12}(\cdot\mid s)=g_{k+\frac12}(s,\cdot)_\#\nu$.
Thereby, for implicit polices, the Wasserstein proximal step \eqref{eq:split_transport} becomes
\[\begin{aligned}
  \max_{g(s,\cdot)} 
  \E_\nu[Q_\tau^{\pi_{k-\frac12}}(s,g(s,Z))] - \frac{1}{2\eta}\wass_2^2(g(s,\cdot), \pi_{k-\frac12}(\cdot\mid s)).
  \end{aligned}
\]
The next proposition shows an exact equivalent reformulation of this problem, as a result of the interchangeability principle, provided that the implicit policy class is expressive.
\begin{proposition}\label{prop:implicit_transport}
Assume the range of the function $g(s,\cdot)$ covers the action space $\cA$. Then \eqref{eq:split_transport} can be equivalently solved via
\begin{equation}
\begin{aligned}
g_{k+\frac12}(s,\cdot)
\in
&~ \argmax_{g(s,\cdot)}
Q_\tau^{\pi_{k-\frac12}}\big(s,g(s,Z)\big)
\\
&\quad -\frac{1}{2\eta}\big\|g(s,Z)-g_{k-\frac12}(s,Z)\big\|^2,
  \end{aligned}
\label{eq:param_g_update}
\end{equation}
with $\pi_{k+\frac12}=g_{k+\frac12}(s,\cdot)_\#\nu$.
\end{proposition}
This result shows that the Wasserstein proximal step for implicit policies can be implemented as a drift-map optimization, where the new drift map $g_{k+\frac12}$ is obtained by maximizing the expected Q-value minus a quadratic penalty that encourages proximity to the previous drift map $g_{k-\frac12}$.
Taking expectation over $s \sim d_\rho^{\pi_{k-\frac12}}$, we obtain a further equivalent form, provided that the visitation distribution has full support and enables the interchangeability principle:
\[\begin{aligned}
  g_{k+\frac12} \in \argmax_g\ \E_{s\sim d_\rho^{\pi_{k-\frac12}},Z\sim\nu} \Big[
  Q_\tau^{\pi_{k-\frac12}}\big(s,g(s,Z)\big)
  \\\quad - \frac{1}{2\eta}\big\|g(s,Z)-g_{k-\frac12}(s,Z)\big\|^2
  \Big]. 
  \end{aligned}
\]

This formulation leads to the following parametric update rule when $g$ is parameterized by $g_\theta$:
\[\begin{aligned}
  \theta_{k+\frac12} \in \argmax_\theta\ \E_{s\sim d_\rho^{\pi_{k-\frac12}},Z\sim\nu} \Big[
  Q_\tau^{\pi_{k-\frac12}}\big(s,g_\theta(s,Z)\big)
  \\\quad - \frac{1}{2\eta}\big\|g_\theta(s,Z)-g_{\theta_{k-\frac12}}(s,Z)\big\|^2
  \Big].
  \end{aligned}
\]
Finally, the entropy step is applied at sampling time by injecting Gaussian noise:
\[
A_{k+1}
= g_{{k+\frac12}}(s,Z) + \sqrt{2\tau\eta} \xi.
\label{eq:sampling_new_policy}
\]
Thereby, the WPPG scheme \eqref{eq:WPPG} admits a practical implementation for implicit policies via the two-step updates above. Crucially, optimizing the generator parameters $\theta$ does not require access to the policy log-density or its gradient, but only the action-gradient of the soft-$Q$ function. In an actor–critic implementation, this can be obtained from a differentiable critic network, and Appendix~\ref{entropy estimation} describes a procedure to estimate the soft-Q function for our implicit policy class.

It is worth contrasting WPPG with related methods. Although SAC \citep{haarnoja2018soft} uses the reparameterization trick, it still requires a policy family with a tractable log-density (e.g., Tanh–Gaussian) and does not directly extend to general pushforward policies whose densities are typically intractable. WPO \citep{pfau_wasserstein_2025}, includes a KL-based projection step that requires access to the policy log-density. In contrast, our update is derived under $\wass_2$ geometry only; making it compatible with expressive implicit stochastic policies.

\section{Convergence Analysis}

This section derives convergence guarantees for \eqref{eq:WPPG}, considering both exact (Section \ref{sec:Q:exact}) and inexact (Section \ref{sec:Q:inexact}) Q-functions. Our analysis follows the roadmap in \cite{lan2023policy}, but we provide special treatment for the Wasserstein metric in place of the Bregman divergence used therein.

\subsection{Exact Q-function}\label{sec:Q:exact}

We begin by stating the main assumptions. 
Let $\cP_2(\R^d)$ denote the space of probability measures on $\R^d$ with finite second moments, equipped with the 2-Wasserstein metric.

\begin{assumption}
\label{as:Bounded}

\textit{(Boundedness)} For every state $s$, the followings hold:
The reward $r(s,a)$ is uniformly bounded,i.e. $R_{\min}\le r(s,a)\le R_{\max}$ and the action space $\mathcal A$ is bounded by $ R_{A}$. 
\end{assumption}

\begin{assumption}
\label{as:entropy bounded}
The differential entropy $-\H^{\pi_0}(\cdot)$ of initial policy $\pi_0(\cdot|s)$ is uniformly lower bounded by $C_0$.
\end{assumption}

\begin{definition}[$T_2$ transportation-information inequality]\label{def:T2}
A probability measure $\mu \in \mathcal{P}_2(\mathbb{R}^d)$ is said to satisfy the
$T_2(\lambda)$ transportation-information inequality (for some $\lambda>0$) if,
for every $\nu \in \mathcal{P}_2(\mathbb{R}^d)$,
\begin{equation}\label{eq:T2}
\mathsf{W}_2^2(\nu,\mu)\ \le\ \frac{2}{\lambda}\,\mathrm{KL}(\nu\|\mu).
\tag{$T_2(\lambda)$}
\end{equation}
\end{definition}

\begin{assumption}[Uniform $T_2(\lambda)$ along the optimization trajectory]\label{as:T2}
There exists a constant $\lambda>0$ such that for all iterations $k\in\mathbb{N}$ and all states $s\in\mathcal S$,
the distribution $\pi_k(\cdot\mid s)$ satisfies the $T_2(\lambda)$ transportation-information inequality.
\end{assumption}

\begin{remark}
The first assumption \ref{as:Bounded} is quite common in practice. For the second assumption \ref{as:entropy bounded}, we only need to choose the initial policy to be a uniform distribution. We will verify Assumption~\ref{as:T2} in Appendix~\ref{subsectionB:T2}.

\end{remark}

Recall our WPPG scheme \eqref{eq:WPPG},
by the soft Bellman optimality conditions, there exists an optimal policy
$\pi^\star$ such that
\begin{equation*}
V_\tau^{\pi^\star}(s)\ge\ V_\tau^{\pi}(s)\qquad \text{for all } s\in\mathcal S, \quad \text{and any} \ \ \pi.
\label{eq:soft-dominance}
\end{equation*}
Hence, optimizing $V^\pi(\cdot)$ state-wise is equivalent to optimizing any strictly
positive weighted average of it. In particular, for any probability weights
$\rho\in\cP(\mathcal S)$ with full support,
\[
\pi^\star\in\argmax_{\pi} \mathbb{E}_{s\sim\rho} \left[V^\pi(s)\right]
\quad\text{s.t.}\quad \pi(\cdot\mid s)\in \cP(\mathcal A),\ \forall s\in\mathcal S.
\]
While the initial distribution $\rho$ can be chosen arbitrarily, 
we follow \citet{lan2023policy} and set $\rho = \nu^\ast$---the stationary distribution induced by the optimal policy $\pi^\ast$---to ease the proof. Notably, our \eqref{eq:WPPGs} algorithm designed to optimize the objective \eqref{defination of J} does not require access to $\nu^\ast$.
We thereby define the objective function as
\begin{equation}\label{defination of J}
J(\pi):=J_{\nu^\ast}(\pi) = \mathbb{E}_{s \sim \nu^\ast}\!\left[ V^\pi(s) \right],
\end{equation}
and our goal is to maximize $J_{\nu^\ast}(\pi)$ over all admissible policies:
\begin{equation}
\max_{\pi}\ J(\pi)
\quad \text{s.t.} \quad
\pi(\cdot \mid s) \in \cP_{\mathcal{A}},\ \forall s \in \mathcal{S}.
\end{equation}

Our main result in this subsection is as follows.

\begin{theorem}\label{thm:convergence}
Suppose Assumption \ref{as:T2} holds and set the step size $\eta_k=\eta =\frac{1}{\gamma\lambda\tau}$. 
Then for any $k \ge 0$, the iterates of \eqref{eq:WPPGs} satisfy
\begin{align*}
&J(\pi_{\ast}) - J(\pi_{k}) + \lambda\tau \mathcal{D}(\pi_{k}, \pi^\ast) \\ &\leq \gamma^k \left[ J(\pi^\ast) - J(\pi_{0}) + \lambda \tau \mathcal{D}(\pi_{0}, \pi^\ast)\right] 
\end{align*}
where $J$ is defined in \eqref{defination of J}, and
\begin{equation*}
\mathcal{D}(\pi_k, \pi^\ast) := 
\mathbb{E}_{s \sim \nu^\ast} \left[ \tfrac{1}{2}\,\wass_2^2\bigl(\pi_k(\cdot|s), \pi^\ast(\cdot|s)\bigr) \right].
\end{equation*}

Consequently, in order to achieve an error of $\mathcal{O}(\varepsilon+\delta)$, 
the required iteration complexity is
\begin{align*}
\mathcal{O}\!\left(\frac{1}{1-\gamma}\,\log \frac{\,J(\pi^\ast) - J(\pi _0) + \lambda \tau \mathcal{D}(\pi_0, \pi^\ast)}{\varepsilon}\right).
\end{align*}
\end{theorem}

\emph{Proof sketch.}
We sketch the key one-step recursion; The detailed proof is deferred to Appendix~\ref{app:exact}

Consider WPPG update~\eqref{eq:WPPGs}.
Let $\varphi^{\pi_{k+1}\to\pi_k}(s,\cdot)$ be an optimal Kantorovich potential for
$\big(\pi_{k+1}(\cdot|s),\pi_k(\cdot|s)\big)$ under quadratic cost.
Kantorovich duality yields the supporting-hyperplane inequality for
$F(q):=\tfrac12\W^2(q,\pi_k(\cdot|s))$:
\begin{align}
\frac12\W^2\!\big(p,\pi_k\big)
 \ge 
\frac12\W^2\!\big(&\pi_{k+1},\pi_k\big)
+\\
&\big\langle \varphi^{\pi_{k+1}\to\pi_k}(s,\cdot),\,p-\pi_{k+1}\big\rangle ,
\label{eq:sketch_w2_support}
\end{align}
for any competitor $p(\cdot|s)$ (we omit $(\cdot|s)$ inside $\W$ for readability).

Since \eqref{eq:WPPGs} is a concave maximization over $q(\cdot|s)\in\cP_{\cA}$,
its first-order optimality condition states that there exists an optimal Kantorovich
potential $\varphi^{\pi_{k+1}\to\pi_k}(s,\cdot)$ for
$\big(\pi_{k+1}(\cdot|s),\pi_k(\cdot|s)\big)$ such that, for all competitors $p(\cdot|s)$,
\begin{align*}
\Big\langle\,
Q_\tau^{\pi_k}(s,\cdot)
    -\tau\big(1+\ln \pi_{k+1}(\cdot|s)\big)
&-\frac{1}{\eta}\varphi^{\pi_{k+1}\to\pi_k}(s,\cdot),\\
&p(\cdot|s)-\pi_{k+1}(\cdot|s)
\Big\rangle
~\le~0 .
\end{align*}
Combining this optimal condition with \eqref{eq:sketch_w2_support} yields the proximal bound:
for all $p(\cdot|s)$,
\begin{align}
\eta\Big(
    \big\langle Q_\tau^{\pi_k}(s,\cdot),\,p-&\pi_{k+1}\big\rangle
    -\tau \H^{p}(s)+\tau \H^{\pi_{k+1}}(s)
\Big)
\\&+\frac12\W^2\!\big(\pi_{k+1},\pi_k\big)
\nonumber\\
&\qquad\le~
\frac12\W^2\!\big(p,\pi_k\big)
-\tau\eta\,\KL\!\big(p\;\|\;\pi_{k+1}\big).
\label{eq:sketch_prox}
\end{align}
Then by Assumption~\ref{as:T2}(which could be verified),
\begin{align}
\KL\!\big(p\;\|\;\pi_{k+1}\big)
~\ge~
\frac{\lambda}{2}\,\W^2\!\big(p,\pi_{k+1}\big),
\label{eq:sketch_T2_convert}
\end{align}
so the last term in \eqref{eq:sketch_prox} yields explicit $\W^2(p,\pi_{k+1})$ control.

Next, we relate the linear term to the value gap under $\nu^\ast$.
The entropy-regularized performance difference lemma and stationarity of $\nu^\ast$ under $\pi^\ast$
imply, for any policy $\pi$,
\begin{align}
&(1-\gamma)\,\E_{s\sim\nu^\ast}\!\big[V_\tau^{\pi^\ast}(s)-V_\tau^{\pi}(s)\big]
=\\
&\E_{s\sim\nu^\ast}\!\Big[
    \big\langle Q_\tau^{\pi}(s,\cdot),\,\pi^\ast(\cdot|s)-\pi(\cdot|s)\big\rangle
    -\tau\H^{\pi^\ast}(s)+\tau\H^{\pi}(s)
\Big].
\label{eq:sketch_steady_transform}
\end{align}
Moreover, applying the entropy-regularized performance difference lemma again with $(\pi',\pi)=(\pi_{k+1},\pi_k)$ and using
\eqref{eq:sketch_prox} with $p=\pi_k(\cdot|s)$ shows the per-state improvement integrand is nonnegative,
yielding the pointwise bound
\begin{align}
V_\tau^{\pi_{k+1}}(s)-V_\tau^{\pi_k}(s)
&~\ge~
\big\langle Q_\tau^{\pi_k}(s,\cdot),\,\pi_{k+1}(\cdot|s)-\pi_k(\cdot|s)\big\rangle
\\&-\tau\H^{\pi_{k+1}}(s)+\tau\H^{\pi_k}(s).
\label{eq:sketch_pointwise_improve}
\end{align}

Instantiate \eqref{eq:sketch_prox} at $p=\pi^\ast(\cdot|s)$, combine
\eqref{eq:sketch_T2_convert}--\eqref{eq:sketch_pointwise_improve}, and take expectation over $s\sim\nu^\ast$.
This yields the one-step recursion
\begin{align*}
J(\pi^\ast)&-J(\pi_{k+1})+\lambda\tau\,\mathcal D(\pi_{k+1},\pi^\ast)
\\~\le~
&\gamma\Big(
J(\pi^\ast)-J(\pi_k)+\tfrac{1}{\eta\gamma}\mathcal D(\pi_k,\pi^\ast)
\Big),
\end{align*}
where $J$ is defined in~\eqref{defination of J} and
$\mathcal D(\pi,\pi^\ast)=\E_{\nu^\ast}\big[\tfrac12\W^2(\pi(\cdot|s),\pi^\ast(\cdot|s))\big]$.
Choosing $\eta=\frac{1}{\gamma\lambda\tau}$ makes the quantity
$J(\pi^\ast)-J(\pi_k)+\lambda\tau\,\mathcal D(\pi_k,\pi^\ast)$
geometrically contract with rate $\gamma$, and iterating yields the theorem.

\begin{remark}
A closely related result is \citet[Theorem~5]{song_provably_2023}. 
They study the unregularized setting with finite action spaces (and a one-dimensional continuous extension), whereas we analyze the entropy-regularized problem on continuous action spaces. 
Technically, their analysis bounds the deviation between the Wasserstein update and classical policy iteration via a uniform bound on pairwise distances between discrete actions, without exploiting the Wasserstein geometry. Consequently, they require an increasing step-size schedule (i.e., decreasing $\beta$), leading to an $O(1/\varepsilon)$ step size to reach $\mathcal{O}(\varepsilon)$ accuracy. 
In contrast, we achieve the same accuracy with a constant step size independent of $\varepsilon$; see Section~\ref{subsection:Cproof theorem}.
\end{remark}

\begin{remark}
Our result is reminiscent of the linear convergence guarantee for mirror descent policy optimization in \citet{lan2023policy}, but the analyses differ in essential ways. 
While \citet{lan2023policy} relies on the KL geometry and the associated three-point (Bregman) lemma, our proof is developed in the $W_2$ geometry, where such a lemma is unavailable. 
Instead, we combine Wasserstein descent estimates with transportation-information inequalities to obtain a linear rate; see details in Section~\ref{subsection:Cproof theorem}.

More importantly, we exploit the structure of the policy trajectory $\{\pi_k\}$: a refined analysis of the corresponding Kantorovich potentials and value functions yields a uniform $T_2$ inequality along the entire trajectory, which is crucial for controlling the Wasserstein proximal term; see details in Section~\ref{subsectionB:T2}.
\end{remark}

\subsection{Inexact Q-function}\label{sec:Q:inexact}
In practice, the exact action-value function $Q^{\pi_k}$ is rarely available, 
since computing it requires either full knowledge of the environment dynamics 
or an infinite number of Monte Carlo samples. Instead, one typically constructs 
a stochastic estimator $Q^{\pi_k, \xi_k}$ from finite trajectories, 
temporal-difference updates, or function approximation. 

In this case, the \eqref{eq:WPPG} update is defined by substituting the exact value function $Q^{\pi_k}$ in \eqref{eq:w2-prox-step} with its stochastic 
estimator $Q^{\pi_k, \xi_k}$. Formally, the update rule is given by
\begin{equation}\label{eq:stochastic_wppg}
\begin{aligned}
\pi_{k+1}(\cdot\mid s)
&\in \arg\max_{q(\cdot\mid s)}
\Big\{
\langle Q^{\pi_k,\xi_k}(s,\cdot),\, q(\cdot\mid s)\rangle
\\&- \tau\, \H^{q}(s) 
- \frac{1}{2\eta_k}\,
\wass_2^2\!\big(q(\cdot\mid s), \pi_k(\cdot\mid s)\big)
\Big\}.
\end{aligned}
\end{equation}

Such an estimator inevitably introduces both variance and bias, which can accumulate across iterations and significantly affect policy updates. To ensure that our analysis remains tractable while still capturing realistic scenarios, we impose mild conditions on the stochastic approximation and estimation error.

\begin{assumption}\label{assumption:stochastic Q}
For each iteration $k \ge 0$, the stochastic estimator $Q^{\pi_k, \xi_k}$ satisfies
\begin{align}
\E_{\xi_k}\left[ Q^{\pi_k, \xi_k} \right] &= \bar{Q}^{\pi_k}, \label{eq:assump1} \\ \big\| \bar{Q}^{\pi_k} - Q^{\pi_k} \big\|_\infty &\le \epsilon_k,\label{eq:assump2}\\\E_{\xi_k}\left[ \big \| \nabla _aQ^{\pi_k, \xi_k} - \nabla_a Q^{\pi_k} \big\|_{2,\infty}^2 \right] &\le \sigma_k^2.        \label{eq:assump3}
\end{align}
Where $\|\cdot\|_\infty$ is the uniform norm over $(s,a) \in \cS \times \cA$, 
and for any function $f$, $\|\cdot\|_{2,\infty}$ is defined as
\[
\| f \|_{2,\infty} := \sup_{(s,a) \in \cS \times \cA} \| f(s,a) \|_2.
\]
\end{assumption}

This assumption is similar to \cite{lan2023policy}, except that \eqref{eq:assump3} concerns the action-gradient of the $Q$-function, which is needed in our Wasserstein policy update.
We have the following convergence result.

\begin{theorem}\label{thm:inexactconvergence}
Suppose Assumptions \ref{as:T2} and \ref{assumption:stochastic Q} hold, and   for all $k \ge 0$, $\epsilon_k\le \epsilon, \sigma_k\le\sigma$ and $Q^{\pi_k, \xi_k}$ are uniformly bounded. Then the iterates of \eqref{eq:stochastic_wppg} using step size $\eta_k=\eta =\frac{1}{\gamma\lambda\tau}$ satisfies
\begin{equation}\label{eq:main_ineq}
\begin{aligned}
\mathbb{E}_{\xi_{0:k-1}}\!\Big[
&J(\pi^\ast) - J(\pi_k) + \lambda\tau\, D(\pi_k,\pi^\ast)
\Big]
 \\
& \le \gamma^k \Big[
J(\pi^\ast) - J(\pi_0) + \lambda\tau\, D(\pi_0,\pi^\ast)
\Big]  + \mathcal{O}(\epsilon+\sigma).
\end{aligned}
\end{equation}

where $J$ is defined in \eqref{defination of J}, and
$\mathcal{D}(\pi_k, \pi^\ast) := 
\mathbb{E}_{s \sim \nu^\ast} \left[ \tfrac{1}{2}\,\wass_2^2\bigl(\pi_k(\cdot|s), \pi^\ast(\cdot|s)\bigr) \right].
$
Consequently, in order to achieve an error of $\mathcal{O}(\varepsilon+\epsilon+\sigma)$ in expectation, 
the required iteration complexity is
\[
\mathcal{O}\!\left(\frac{1}{1-\gamma}\,\log \frac{\,J(\pi^\ast) - J(\pi_0) + \lambda \tau \mathcal{D}(\pi_0, \pi^\ast)}{\varepsilon}\right).
\]
\end{theorem}

The detailed proof is deferred to Appendix~\ref{app:Inexact}.

Compared with Theorem \ref{thm:convergence}, this result involves additional bias arising from the estimation of the $Q$-function and variance from stochastic estimation of the $Q$-function.
Nonetheless, our analysis shows that
the total error does not accumulate across iterations, and
the convergence guarantee only incurs an $\mathcal{O}(\epsilon+\sigma)$ term in the final bound, 
rather than growing with the number of iterations.

\begin{figure*}[ht!]
    \centering
    \includegraphics[width=0.9\linewidth]{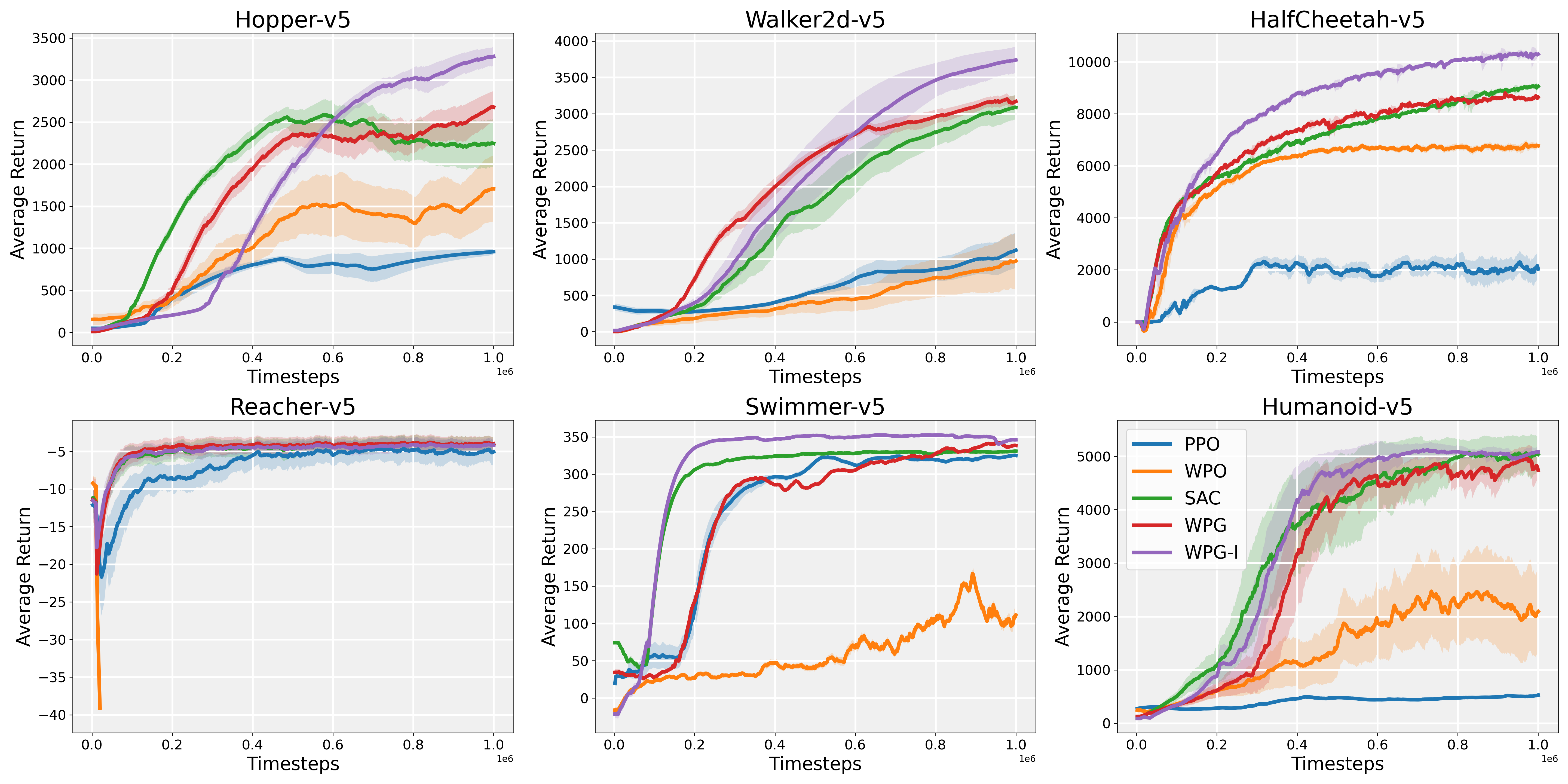}
    \caption{Training curves on MuJoCo continuous control benchmarks: Solid lines denote the mean episodic return, while shaded areas represent the 95\% confidence interval computed over 10 independent evaluation runs with different random seeds.}
    \label{fig:mujoco_grid}
\end{figure*}

\section{Experiments}

Our empirical study consists of two parts: comparative evaluation and ablation analysis. 
The comparative evaluation focuses on benchmarking our methods against representative baselines 
to assess overall performance. The ablation analysis, which is postponed to the appendix, investigates three key questions: 
(i) the effect of $\tau$ on WPPG, 
(ii) the impact of latent variable dimension on WPPG-I, 
and (iii) the role of double-$Q$ learning in WPPG.

\paragraph{Evaluation Tasks}
We evaluate our approach on a set of standard continuous control benchmarks from the MuJoCo suite \footnote{https://gymnasium.farama.org/environments/mujoco/}, including Hopper-v5, Walker2d-v5, HalfCheetah-v5, Reacher-v5, Swimmer-v5, and Humanoid-v5.These tasks cover a wide range of difficulties: from relatively low-dimensional and easy-to-learn tasks such as Swimmer and Hopper, to high-dimensional and challenging tasks such as Humanoid. 
\paragraph{Baseline Models}
We compare against three representative baselines: 
(i) PPO, a KL-proximal policy optimization method that employs clipped surrogate objectives to constrain successive policy updates and improve training stability \citep{schulman_proximal_2017};
(ii) SAC, a stochastic actor–critic algorithm formulated as an entropy-regularized policy optimization method, which augments the reward with a maximum-entropy term to encourage exploration and improve stability \citep{haarnoja2018soft}.
(iii) WPO, a Wasserstein-proximal actor–critic algorithm that replaces the KL divergence commonly used in proximal methods with the Wasserstein distance, thereby constraining successive policy updates under the geometry of optimal transport. \citep{pfau_wasserstein_2025}. 
Our proposed methods include WPPG, with a Gaussian MLP policy actor, and WPPG-I, with an implicit MLP policy actor. 
Hyperparameters for PPO are taken from the RL Zoo project\footnote{https://github.com/DLR-RM/rl-baselines3-zoo}, while those for SAC and WPO follow \citet{pfau_wasserstein_2025}. 

\paragraph{Experiment Setup}
For fair comparison, WPPG and WPPG-I adopt the double-$Q$ technique, taking the minimum of two $Q$-functions for both critic targets and action gradients, and WPO follows the original implementation provided in ACME\footnote{https://github.com/google-deepmind/acme}. To further ensure fairness, SAC is evaluated with entropy coefficient self-tuning disabled, since $\tau$ in WPPG and WPPG-I is also fixed rather than adaptively tuned while training. All off-policy methods share the same replay buffer structure. Additional implementation details for all methods and pseudo code of our algorithm are provided in the Appendix \ref{Numerical}.

\paragraph{Results and Discussion}
The learning curves for all tasks are shown in Figure \ref{fig:mujoco_grid}. 
Across the six MuJoCo benchmarks, WPPG demonstrates performance comparable to SAC, 
which can be attributed to the fact that both methods employ Gaussian MLP policies. 
This observation suggests that Wasserstein geometry can match, and in some cases even surpass, 
the effectiveness of KL-based geometry in policy optimization. 
More importantly, WPPG-I consistently outperforms all baselines, achieving higher returns across nearly all tasks. 
The success of WPPG-I further indicates that WPPG can be naturally extended to implicit policy classes: 
although we use only a simple MLP-based implicit policy here, the framework readily accommodates 
richer architectures. 
We refer readers to the Appendix \ref{Algorithm} and \ref{Implementation} for a detailed comparison between the two algorithms. 
In contrast, PPO lags behind due to slower learning and lower asymptotic performance, while WPO suffers 
from unstable convergence on challenging environments such as Humanoid and Swimmer and even fails to 
learn in Reacher. 
Overall, these results highlight that WPPG preserves the sample efficiency of off-policy actor-critic methods, 
while WPPG-I not only inherits these advantages but also demonstrates a consistent and significant margin over all baselines.

\section{Conclusion}

In this work, we proposed Wasserstein Proximal Policy Gradient (WPPG), a novel framework for policy optimization that leverages Wasserstein geometry to design proximal updates directly in distribution space. Our method eliminates the need for policy densities or score functions, making it naturally applicable to implicit policies. Theoretically, we established linear convergence guarantees under entropy regularization and mild conditions, covering both exact and approximate value function settings. To our best knowledge, this work is an early attempt to employ Wasserstein geometry for establishing global convergence guarantees. Empirically, WPPG with implicit policies demonstrates superior performance in challenging continuous-control benchmarks.

\section*{Impact Statement}
This paper presents work whose goal is to advance the field of Machine
Learning theory. There are many potential societal consequences of our work, none of which we feel must be specifically highlighted here.

\bibliography{example}
\bibliographystyle{icml2026}

\newpage
\appendix
\onecolumn

\appendix

\section{Proof of Proposition \ref{prop:implicit_transport}}\label{app:transport}

Fix a state $s$. Abbreviate $Q(a):=Q_\tau^{\pi_{k-\frac12}}(s,a)$ and $\mu:=\pi_{k-\frac12}(\cdot\mid s)$.

\begin{proof}

Write $\mu:=\pi_{k-\frac12}(\cdot\mid s)$.
Consider the proximal problem
\begin{equation}
\label{eq:state_prox_pi}
\sup_{\pi}\Big\{\langle Q,\pi\rangle-\tfrac1{2\eta}\wass_2^2(\pi,\mu)\Big\}.
\end{equation}
By the definition of $\wass_2$,
\[
W_2^2(\pi,\mu)=\inf_{\gamma\in\Gamma(\pi,\mu)}\int \|a-b\|^2\,\gamma(da,db).
\]
Hence, merging the negative inf with outer sup of \eqref{eq:state_prox_pi} and using the fact that $\langle Q,\pi\rangle=\int Q(a)\,\gamma(da,db)$ for any
$\gamma\in\Gamma(\pi,\mu)$, we can rewrite \eqref{eq:state_prox_pi} as
\begin{equation}
\label{eq:gamma_fixed_second}
\sup_{\gamma:\ \gamma_2=\mu}\ \int\Big(Q(a)-\tfrac1{2\eta}\|a-b\|^2\Big)\,\gamma(da,db),
\end{equation}
where $\gamma_B$ denotes the second marginal of $\gamma$ (so the second marginal is fixed to be $\mu$).

Disintegrate $\gamma(da,db)=\gamma(da\mid b)\,\mu(db)$ to obtain
\[
\sup_{\gamma(\cdot\mid b)}\ \int\left[\int\Big(Q(a)-\tfrac1{2\eta}\|a-b\|^2\Big)\,\gamma(da\mid b)\right]\mu(db).
\]
For each fixed $b$, the inner term is maximized by a Dirac mass at any maximizer of the integrand:
\[
T_s(b)\in\arg\max_{a\in\cA}\left\{Q(a)-\tfrac1{2\eta}\|a-b\|^2\right\},
\]
which is assumed to exist; otherwise we can consider an $\epsilon$-optimal selection.
Then an optimal coupling for \eqref{eq:gamma_fixed_second} is $\gamma^\star=(T_s,\mathrm{Id})_\#\mu$, and the corresponding optimal policy is
\begin{equation}
\label{eq:pi_pushforward}
\pi_{k+\frac12}(\cdot\mid s)= (T_s)_\#\mu.
\end{equation}
Moreover, the optimal value admits the envelope representation
\begin{equation}
\label{eq:envelope_b}
\sup_{\pi}\Big\{\langle Q,\pi\rangle-\tfrac1{2\eta}W_2^2(\pi,\mu)\Big\}
=
\int \sup_{a}\left\{Q(a)-\tfrac1{2\eta}\|a-b\|^2\right\}\,\mu(db).
\end{equation}

Finally, use the implicit representation $\mu=(g_{k-\frac12}(s,\cdot))_\#\nu$:
if $Z\sim\nu$ and $B=g_{k-\frac12}(s,Z)$ then $B\sim\mu$, so \eqref{eq:envelope_b} becomes
\begin{equation}
\label{eq:envelope_z}
\mathbb E_{Z\sim\nu}\left[\sup_{a}\left\{Q(a)-\tfrac1{2\eta}\|a-g_{k-\frac12}(s,Z)\|^2\right\}\right].
\end{equation}
By interchangeability, \eqref{eq:envelope_z} equals
\[
\sup_{g(s,\cdot)}\ \mathbb E_{Z\sim\nu}\left[
Q(g(s,Z))-\tfrac1{2\eta}\|g(s,Z)-g_{k-\frac12}(s,Z)\|^2
\right],
\]
and an optimizer is realized by the pointwise choice
\[
g_{k+\frac12}(s,Z)=T_s\!\big(g_{k-\frac12}(s,Z)\big),
\]
which is consistent with \eqref{eq:pi_pushforward} since then
$(g_{k+\frac12}(s,\cdot))_\#\nu=(T_s)_\#(g_{k-\frac12}(s,\cdot))_\#\nu=\pi_{k+\frac12}(\cdot\mid s)$.
This completes the proof.
\end{proof}

\section{Entropy flow and the heat semigroup}\label{app:entropy-heat}

We justify the closed form \eqref{eq:split_heat}.

\begin{proposition}[Entropy $W_2$-flow is the heat equation]\label{prop:entropy_heat}
Let $q_t(da)=\rho_t(a)\,da$ be an absolutely continuous curve in $\mathcal P_2(\mathbb R^d)$.
Consider the (negative) Boltzmann entropy
\[
\H(\rho)=\int_{\mathbb R^d}\rho(a)\log\rho(a)\,da,
\]
and the functional $\Phi_2(\rho)=\tau \mathcal H(\rho)$.
Then the Wasserstein gradient flow of $\Phi_2$ is the heat equation
\[
\partial_t \rho_t = \tau \Delta \rho_t.
\]
Its solution operator is Gaussian convolution:
\[
\rho_t = \rho_0 * \mathcal N(0,2\tau t\,I_d).
\]
Equivalently, if $X_0\sim\rho_0$ and $\xi\sim\mathcal N(0,I_d)$ are independent, then
$X_t:=X_0+\sqrt{2\tau t}\,\xi$ satisfies $\Law(X_t)=q_t$.
\end{proposition}

\begin{proof}
Formally, the Wasserstein gradient flow of an internal energy functional
$\int f(\rho)$ satisfies
\[
\partial_t \rho_t = \nabla\cdot\!\left(\rho_t\nabla\frac{\delta \Phi_2}{\delta\rho}(\rho_t)\right).
\]
For $\Phi_2(\rho)=\tau\int \rho\log\rho$, the first variation is
$\delta\Phi_2/\delta\rho=\tau(1+\log\rho)$.
Thus
\[
\partial_t\rho_t=\nabla\cdot\big(\rho_t\nabla(\tau(1+\log\rho_t))\big)
=\tau\nabla\cdot(\rho_t\nabla\log\rho_t)
=\tau\nabla\cdot(\nabla\rho_t)
=\tau\Delta\rho_t.
\]
The heat equation on $\mathbb R^d$ has the explicit solution $\rho_t=\rho_0*\mathcal N(0,2\tau t I)$.
The probabilistic representation follows since convolution with a Gaussian equals adding an independent Gaussian random variable.
\end{proof}

\section{Theoretical Derivations and Proofs}

\subsection{G\^{a}teaux Derivative of Entropy-rugularized Reinforcement Learning Objective}

\begin{lemma}[Log-derivative of the trajectory law]\label{lem:log-derivative}
Consider a zero-mass perturbation $\delta\pi(\cdot|s)$ (i.e.\ $\int_{\mathcal A}\delta\pi(da|s)=0$ for all $s$) and
the policy path $\pi_\epsilon=\pi+\epsilon\,\delta\pi$, with $|\epsilon|$ small so that $\pi_\epsilon\ge0$ and
$\mathrm{supp}(\delta\pi)\subseteq\{\pi>0\}$.
Let $\P_{\pi_\epsilon}$ be the path measure of the Markov chain induced by $\pi_\epsilon$ and the transition kernel $P(\d s'|s,a)$
from an initial law $\rho_0$. For any finite horizon $T$ and any integrable test functional $F$ of $\tau_{0:T}=(s_0,a_0,\dots,s_T)$,
\begin{equation}\label{eq:log-derivative-finite}
\left.\frac{d}{d\epsilon}\right|_{\epsilon=0}
\int F(\tau_{0:T}) \, dP_{\pi_\epsilon}(\tau_{0:T})
= \int F(\tau_{0:T})
\left( \sum_{u=0}^{T-1} \frac{\delta \pi(a_u \mid s_u)}{\pi(a_u \mid s_u)} \right)
\, dP_{\pi}(\tau_{0:T}).
\end{equation}
If $F$ is dominated by an integrable function uniformly, then letting $T\to\infty$ and applying
dominated convergence yields
\begin{equation}\label{eq:log-derivative-infinite}
\left.\frac{d}{d\epsilon}\right|_{\epsilon=0}
\int F(\tau)\, dP_{\pi_\epsilon}(\tau)
= \int F(\tau)\left( \sum_{u=0}^{\infty} \frac{\delta \pi(a_u \mid s_u)}{\pi(a_u \mid s_u)} \right)\, dP_{\pi}(\tau).
\end{equation}
\end{lemma}

\begin{proof}
Write the truncated path density under $\pi_\epsilon$ as
\[
\mathrm d \mathbb{P}_{\pi_\epsilon}(\tau_{0:T})
= \rho_0(s_0)\prod_{t=0}^{T-1}\Big[\pi_\epsilon(a_t \mid s_t)\,P(s_{t+1}\mid s_t,a_t)\Big].
\]
Only $\pi_\epsilon$ depends on $\epsilon$, hence by the product rule,
\[
\frac{\mathrm d}{\mathrm d \epsilon}\prod_{t=0}^{T-1}\pi_\epsilon(a_t \mid s_t)
= \Big(\prod_{t=0}^{T-1}\pi_\epsilon(a_t \mid s_t)\Big)
\sum_{u=0}^{T-1}\frac{\frac{\mathrm d}{\mathrm d \epsilon}\,\pi_\epsilon(a_u \mid s_u)}{\pi_\epsilon(a_u \mid s_u)}.
\]
Evaluating at $\epsilon=0$ gives
\[
\left.\frac{\mathrm d}{\mathrm d\epsilon}\,\pi_\epsilon(a_u \mid s_u)\right|_{\epsilon=0}
= \delta\pi(a_u \mid s_u),
\]
hence \eqref{eq:log-derivative-finite} follows by dominated convergence under the stated integrability.
\end{proof}

\begin{lemma}[Functional policy gradient of Entropy]\label{lem:entropy-gateaux}
Fix a state $s\in\cS$. Let the policy at $s$ admit a density $\pi(\cdot|s)$ w.r.t.\ a reference measure $\d a$
(and write $h^\pi(s):=h(\pi(\cdot|s))$) with
\[
\H^{\pi}(s)\;=\;\!\int_{\cA} \pi(a|s)\,\log\pi(a|s)\,\d a.
\]
For any zero-mass direction $\delta\pi(\cdot|s)$ (i.e.\ $\int_{\cA}\delta\pi(a|s)=0$) and the perturbation
$\pi_\epsilon(\cdot|s)=\pi(\cdot|s)+\epsilon\,\delta\pi(\cdot|s)$, provided $\pi_\epsilon(\cdot|s)$ stays nonnegative for
$|\epsilon|$ small and $\pi(a|s)>0$ on its support, the G\^ateaux derivative of $h^\pi(s)$ in the direction $\delta\pi(\cdot|s)$ is
\begin{equation}\label{eq:entropy-variation}
\left.\frac{\delta \H^{\pi}(s)}{\delta\pi} := \frac{\mathrm d}{\mathrm d\epsilon}\right|_{\epsilon=0} \H^{\pi_\epsilon}(s)
\;=\; \int_{\cA} 1+\log \pi(a|s)\;\delta\pi(\mathrm d a|s).
\end{equation}
\end{lemma}

\begin{proof}
Let $\pi_\epsilon(a)=\pi(a|s)+\epsilon\,\delta\pi(a|s)$. Then
\begin{align*}
\lim_{\epsilon \to 0} \frac{\H^{\pi_\epsilon}(s)-\H^{\pi}(s)}{\epsilon} &=\lim_{\epsilon \to 0} \frac{1}{\epsilon}(\!\int \pi_\epsilon(a|s)\,\log \pi_\epsilon(a|s)\,\mathrm d a - \!\int \pi(a|s)\,\log \pi(a|s)\,\mathrm d a)\\
& =\lim_{\epsilon \to 0} \frac{1}{\epsilon}(\!\int (\pi(a|s)+\epsilon\,\delta\pi(a|s))\,\log (\pi(a|s)+\epsilon\,\delta\pi(a|s))\,\mathrm d a - \!\int \pi(a|s)\,\log \pi(a|s)\,\mathrm d a)\\
&=\int \lim_{\epsilon \to 0}\frac{\pi(a|s)\bigl(\ln( \pi(a|s)+\epsilon\,\delta\pi(a|s)) -\ln \pi(a|s)\bigr)}{\epsilon} + \delta\pi(a|s)\ln(\pi(a|s)+\epsilon \delta\pi(a|s)) \mathrm d a\\
&=\int (1+\ln\pi) \mathrm d \delta\pi
\end{align*}
which implies that $\frac{\delta \H^{\pi}(s)}{\delta\pi}(a)= 1+\ln\pi(a|s) $.
\end{proof}

\begin{lemma}[Functional policy gradient under entropy regularization]\label{lem:functional-gradient}
Then for any zero-mass direction $\delta\pi$, the G\^ateaux derivative (functional gradient) 
\begin{equation}\label{eq:gateaux-kernel}
\frac{\delta V_\tau^\pi(s_0)}{\delta \pi}(s,a)
=\frac{1}{1-\gamma}\,d_{s_0}^\pi(s)\,\Big(Q_\tau^\pi(s,a)- \tau(1+ \ln\pi(a|s)) \Big),
\end{equation}
\end{lemma}

\begin{proof}
We should be very careful here: when computing the functional gradient of $V_\tau^\pi(s)$
with respect to $\pi$, there are two contributing parts - one from the MDP dynamics and one from the entropy regularization.
Let $V_\tau^\pi(s):=\E_{\pi,s_0=s}\left[\sum_{t\geq 0}\gamma^t\big(r(s_t,a_t)- \tau \H^\pi(s_t)\big)\right]$
By Lemma~\ref{lem:log-derivative} with $F(\tau_{})=\sum_{t\geq0}\gamma^t\big(r(s_t,a_t)-\tau \H^\pi(s_t)\big)$,
\begin{align*}
\left.\frac{\mathrm d}{\mathrm d\epsilon}\right|_{\epsilon=0}\!V_\tau^{\pi_\epsilon}
&= \sum_{t\geq0}\gamma^t\,\mathbb E_\pi\!\left[\big(r(s_t,a_t)-\tau\H^\pi(s_t)\big)\sum_{u=0}^{t}\frac{\delta\pi(a_u|s_u)}{\pi(a_u|s_u)}\right]
\;-\; \sum_{t\geq 0}\gamma^t\,\mathbb E_\pi\!\left[    \left.      \frac{\mathrm d}{\mathrm d\epsilon}\right|_{\epsilon=0} \tau\H^{\pi_\epsilon}(s)     \right] \\
\end{align*}

\emph{(I) Trajectory-law part.}  
\[
\mathrm{(I)}
=\sum_{u\ge0}\E_\pi\!\left[\frac{\delta\pi(a_u|s_u)}{\pi(a_u|s_u)}
\sum_{t\ge u}\gamma^t\big(r(s_t,a_t)-\tau\H^\pi(s_t)\big)\right]
=\sum_{u\ge0}\gamma^u\,\E_\pi\!\left[\frac{\delta\pi(a_u|s_u)}{\pi(a_u|s_u)}\,Q_\tau^\pi(s_u,a_u)\right].
\]
Condition on $s_u$ and expand over $a$ to cancel $\pi$, then sum over time and use
$\sum_{u\ge0}\gamma^u\,\Prob_\pi(s_u =  s\mid s_0)=\tfrac{1}{1-\gamma}\,d_{s_0}^\pi(s)$ to get
\[
\mathrm{(I)}
=\iint_{\cS\times\cA}\frac{1}{1-\gamma}\,d_{s_0}^\pi(\d s)\,Q_\tau^\pi(s,a)\,\delta\pi( a|s) \d s \d a.
\]

\emph{(II) Explicit entropy part.}  
By Lemma~\ref{lem:entropy-gateaux} for each state $s$, \(\H^\pi(s)=-\int \pi(a|s)\log\pi(a|s)\,\d a\) , we have
\[
\left.      \frac{\mathrm d}{\mathrm d\epsilon}\right|_{\epsilon=0} \H^{\pi_\epsilon}(s)  =\int_{\mathcal A}\,[1+\log\pi(a|s)]\,\delta\pi(a|s)\mathrm d a,
\]
hence
\[
\mathrm{(II)}=\sum_{t\ge0}\gamma^t\,\mathbb E_\pi[\tau \left.      \frac{\mathrm d}{\mathrm d \epsilon}\right|_{\epsilon=0} \H^{\pi_\epsilon}(s)  ]
=\iint_{\mathcal S\times \mathcal A}\frac{1}{1-\gamma}\,d_{s_0}^\pi( s)\,\tau[1+\log\pi(a|s)]\,\delta\pi(a|s) \mathrm d a \mathrm d s.
\]
Finally, we have $$\frac{\delta V_\tau^{\pi}(s_0)}{\delta\pi} (s,a)= \frac{1}{1-\gamma}\,d_{s_0}^\pi(s)\bigl( Q_\tau^{\pi}(s,a) - \tau(1+\ln \pi(a|s))  \bigr) $$
\end{proof}

\subsection{Boundedness of the Value Function and Per-Iteration Satisfaction of the \texorpdfstring{$T_2$}{T2} Inequality}\label{subsectionB:T2}

\begin{lemma}[Entropy upper bound from bounded action space]\label{lem:entropy_bounded_action}
Let $A\in\mathbb{R}^d$ be absolutely continuous with differential entropy $h(A) := -\int_{\mathcal{A}} p(a)\ln p(a)\,da$.
Assume $A$ is supported on a measurable set $\mathcal{A}\subset\mathbb{R}^d$ with
$0<\mathrm{Vol}(\mathcal{A})<\infty$, i.e.,
\[
\mathbb{P}(A\in\mathcal{A})=1.
\]
Then
\[
h(A)\;\le\;\ln \mathrm{Vol}(\mathcal{A}).
\]
In particular, if $\mathcal{A}\subseteq B_2(0,R_A)$ (the Euclidean ball of radius $R_A$),
then
\[
h(A)\;\le\;\ln\!\Big(\mathrm{Vol}(B_2(0,R_A))\Big)
=\ln\!\Big(\frac{\pi^{d/2}}{\Gamma\!\big(\frac d2+1\big)}\,R_A^d\Big).
\]
If $\mathcal{A}\subseteq [-R_A,R_A]^d$, then
\[
h(A)\;\le\; d\ln(2R_A).
\]
\end{lemma}

\begin{proof}
Let $p$ be the density of $A$ with respect to Lebesgue measure, and note that
$p(a)=0$ for all $a\notin\mathcal{A}$.
Let $u$ denote the uniform density on $\mathcal{A}$:
\[
u(a)=\frac{1}{\mathrm{Vol}(\mathcal{A})}\,\mathbf{1}_{\mathcal{A}}(a).
\]
Consider the KL divergence $\mathrm{KL}(p\|u)$, which is nonnegative:
\[
0\le \mathrm{KL}(p\|u)
=\int_{\mathcal{A}} p(a)\ln\frac{p(a)}{u(a)}\,da
=\int_{\mathcal{A}} p(a)\ln p(a)\,da + \ln \mathrm{Vol}(\mathcal{A}),
\]
where we used $\log u(a)=-\ln \mathrm{Vol}(\mathcal{A})$ on $\mathcal{A}$ and
$\int_{\mathcal{A}} p(a)\,da=1$.
Rearranging yields
\[
-\int_{\mathcal{A}} p(a)\ln p(a)\,da \le \ln \mathrm{Vol}(\mathcal{A}),
\]
i.e.,
\[
h(A)= -\int_{\mathbb{R}^d} p(a)\ln p(a)\,da
= -\int_{\mathcal{A}} p(a)\ln p(a)\,da
\le \ln \mathrm{Vol}(\mathcal{A}).
\]
For the explicit bounds, note that if $\mathcal{A}\subseteq B_2(0,R_A)$ then
$\mathrm{Vol}(\mathcal{A})\le \mathrm{Vol}(B_2(0,R_A))$, and similarly if
$\mathcal{A}\subseteq [-R_A,R_A]^d$ then $\mathrm{Vol}(\mathcal{A})\le (R_A)^d$.
Substituting these volume upper bounds completes the proof.
\end{proof}

In the next part, we will prove that the value function are uniformly bounded and then proof the policy in the trajectory satisfy $T_2$ Uniformly.

Recall our definition of value function
\begin{equation}\label{eq: soft value} 
\begin{aligned}
\\
&Q_\tau^\pi(s,a) = r(s,a)-\tau \H^\pi(s)+\gamma\E\left[V_\tau^\pi(s')\,\middle|\,s,a\right],\\
&V_\tau^\pi(s)=\E_{a\sim\pi(\cdot|s)}\left[Q_\tau^\pi(s,a)\right]. 
\end{aligned}
\end{equation}

For ease of presentation in this part, we slightly modify the notation in this section. Define

\begin{equation}\label{eq: soft value k notation} 
\begin{aligned}
\\
&Q_k(s,a) := r(s,a)+\gamma\E\left[V_k(s')\,\middle|\,s,a\right],\\
&V_k(s) :=\E_{a\sim\pi_k(\cdot|s)}\left[Q_k(s,a)\right]-\tau \H^\pi_k(s). 
\end{aligned}
\end{equation}

It is straightforward to verify that
\begin{equation}\label{eq: soft value relation} 
\begin{aligned}
\\
&Q_\tau^{\pi_k}(s,a) = Q_k(s,a) - \tau \H^{\pi_k}(s)\\
&V_\tau^{\pi_k}=V_k. 
\end{aligned}
\end{equation}

In particular, the value functions coincide, and 
$Q_\tau^{\pi_k}(s,\cdot)$ differs from $Q_k(s,\cdot)$ only by an additive constant (depending on s but not on the action).

Recall the policy update:
\begin{equation}\label{eq:refwppg}
\pi_{k+1}(\cdot\mid s)\in\argmax_{q(\cdot|s)\in\mathcal P(\mathcal A)}
\Big\{
\langle Q_{\tau}^{\pi_k}(s,\cdot),q(\cdot|s)\rangle
- \tau \H^{q}(s)
- \frac{1}{2\eta_k}\,W_2^2\big(q(\cdot|s),\pi_k(\cdot\mid s)\big)
\Big\}.
\end{equation}

\begin{equation}\label{eq:refnewvalueupadate}
\pi_{k+1}(\cdot\mid s)\in\argmax_{q(\cdot|s)\in\mathcal P(\mathcal A)}
\Big\{
\langle Q_k(s,\cdot),q(\cdot|s)\rangle
- \tau \H^{q}(s)
- \frac{1}{2\eta_k}\,W_2^2\big(q(\cdot|s),\pi_k(\cdot\mid s)\big)
\Big\}.
\end{equation}

The solutions of \eqref{eq:refwppg} and \eqref{eq:refnewvalueupadate} are the same since $Q_\tau^{\pi_k}(s,\cdot)$ differs from $Q_k(s,\cdot)$ only by an additive constant (depending on s but not on the action). 

In the sequel, we will use $Q_k$ and $V_k$ as our notation.

\begin{lemma}[Monotone improvement]\label{lem:mono_polished}
For all $k\ge 0$ and all states $s$, $V_{k+1}(s)\ge V_{k}(s)$.
\end{lemma}

\begin{equation}\label{eq:refwppg2}
\pi_{k+1}(\cdot\mid s)\in\argmax_{q(\cdot|s)\in\mathcal P(\mathcal A)}
\Big\{
\langle Q_{k}(s,\cdot),q(\cdot|s)\rangle
- \tau \H^{q}(s)
- \frac{1}{2\eta_k}\,W_2^2\big(q(\cdot|s),\pi_k(\cdot\mid s)\big)
\Big\}.
\end{equation}

\begin{lemma}[Monotone improvement]\label{lem:mono_softvalue}
For all $k\ge 0$ and all states $s$, $V_{\tau}^{\pi_{k+1}}(s)\ge V_{\tau}^{\pi_{k}}(s)$.
\end{lemma}

\begin{proof}
Fix $s$. By optimality of $\pi_{k+1}(\cdot\mid s)$ in \eqref{eq:refwppg2},
evaluating the objective at $q=\pi_k(\cdot\mid s)$ gives
\[
\E_{\pi_{k+1}(\cdot\mid s)}[Q_k(s,a)]
-\tau\H^{\pi_{k+1}}(s)
\ge
\E_{\pi_k(\cdot\mid s)}[Q_{_k}(s,a)]
-\tau\H^{\pi_k}(s)
=V_k(s).
\]
Let $T_\tau^\pi$ denote the regularized policy evaluation operator
\[
(T_\tau^{\pi}V)(s)
:=\E_{a\sim \pi(\cdot\mid s)}\Big[r(s,a)+\gamma\,\E_{s'\sim p(\cdot\mid s,a)}[V(s')]\Big]-\tau\H^{\pi}(s)
.
\]
Then the previous display is exactly $(T_\tau^{\pi_{k+1}}V_k)(s)\ge V_k(s)$.
Since $T_\tau^{\pi_{k+1}}$ is monotone and a $\gamma$-contraction in $\|\cdot\|_\infty$
(see e.g. \citep[Ch.~6]{Puterman1994}),
iterating yields
\[
V_{k+1} = \lim_{n\to\infty}(T_\tau^{\pi_{k+1}})^n V_k \ge V_k.
\]
\end{proof}

\begin{lemma}[Uniform bounds on $Q_k$]\label{lem:Qbounds}
Under Assumptions \ref{as:Bounded}\ref{as:entropy bounded}, for all $k\ge 0$, $s$, and $a$,
\[
\sup_{a\in\cA}Q_k(s,a) \le \frac{R_{\max}+ \gamma\tau \ln \mathrm{Vol(\mathcal{A})}}{1-\gamma},
\qquad
\inf_{a\in\cA}Q_k(s,a) \ge \frac{R_{\min}+\gamma\tau C_0}{1-\gamma}.
\]
Consequently,
\begin{equation}\label{eq:DeltaQ_polished}
\osc\big(Q_k(s,\cdot)\big) \le \frac{R_{\max}-R_{\min}+\gamma\tau \ln \mathrm{Vol(\mathcal{A})}-\gamma\tau C_0}{1-\gamma} =: \Delta_Q.
\end{equation}
\end{lemma}

\begin{proof}
For the upper bound, use $r\le R_{\max}$ and $V_k\le \frac{R_{\max}+\tau\log\mathrm{Vol}(\cA)}{(1-\gamma)}$ since the entropy term term in \eqref{eq:refnewvalueupadate} have an upper bound \ref{lem:entropy_bounded_action}:
\[
Q_k(s,a)=r(s,a)+\gamma \E[V_k(s')] \le R_{\max}+\gamma\frac{R_{\max}+\tau\log\mathrm{Vol}(\cA)}{1-\gamma}
=\frac{R_{\max}+\gamma\tau\log\mathrm{Vol}(\cA)}{1-\gamma}.
\]
For the lower bound, note that by Assumption \ref{as:entropy bounded},
\[
V_0(s)=\E_{\pi_0}[Q_0(s,a)]-\tau\H(\pi_0(\cdot\mid s)\|p_0)
\ge R_{\min}+\gamma\inf_{s'}V_0(s')+\tau C_0,
\]
so $\inf_s V_0(s)\ge (R_{\min}+\tau C_0)/(1-\gamma)$. By Lemma~\ref{lem:mono_polished},
$V_k\ge V_0$, hence
\[
Q_k(s,a)=r(s,a)+\gamma\,\E[V_k(s')]\ \ge\ R_{\min}+\gamma\inf_{s'}V_0(s')
=\frac{R_{\min}+\gamma\tau C_0}{1-\gamma}.
\]
Subtracting gives \eqref{eq:DeltaQ_polished}.
\end{proof}

\begin{lemma}[Uniform-in-time $\H^\pi$ bound]\label{lem:H bound}
Under assumption \ref{as:Bounded} and \ref{as:entropy bounded}, we know for all k,s:
\begin{equation*}
   -\ln\mathrm{Vol}(\mathcal{A}) \le\H^{\pi_k}(s)\le 
\frac{R_{\max}-R_{\min}}{\tau(1-\gamma)}
+\frac{\ln\mathrm{Vol}(\mathcal{A})-C_0}{1-\gamma}
\end{equation*}
\end{lemma}
\begin{proof}
Using \eqref{eq:refnewvalueupadate},
\begin{equation}
\H^{\pi_k}(s)
=
\frac{ \E_{a\sim \pi_k(\cdot\mid s)}[Q_k(s,a)] - V_k(s)}{\tau}
\le
\frac{ \sup_a Q_k(s,a) - V_k(s)}{\tau}.
\end{equation}
From the bounded-reward assumption and because the entropy regularized term is always upper bounded, we always have
From \ref{lem:Qbounds}, we have for all $k\ge 0$, $s$, and $a$,
\[
\sup_{a\in\cA}Q_k(s,a) \le \frac{R_{\max}+ \gamma\tau \ln \mathrm{Vol(\mathcal{A})}}{1-\gamma},
\qquad
\inf_{a\in\cA}Q_k(s,a) \ge \frac{R_{\min}+\gamma\tau C_0}{1-\gamma}.
\]
Using $V_k(s)\ge V_0(s)$ and $Q_k$ is uniformly upper bounded gives
\begin{equation}
\H^{\pi_k}(s)\le
\frac{ \frac{R_{\max}+\gamma\tau \mathrm{Vol(\mathcal{A})}}{1-\gamma} - V_0(s)}{\tau}
\quad \forall k,\ \forall s.
\end{equation}
\end{proof}

To get a uniform lower bound on $V_0(s)$, we have assumed \ref{as:entropy bounded} $C_0 := \inf_s -\H^{\pi_0}(s) > -\infty$.
Then
\[
V_0(s) \ge
\frac{R_{\min}}{1-\gamma}+\frac{\tau C_0}{1-\gamma}.
\]
Plugging this into the previous bound yields
\begin{equation}
\sup_{k,s}\H^{\pi_k}(s)
\le 
\frac{R_{\max}-R_{\min}}{\tau(1-\gamma)}
+\frac{\ln\mathrm{Vol}(\mathcal{A})-C_0}{1-\gamma}.
\end{equation}

and with \ref{lem:entropy_bounded_action}, we know for all k,s:
\begin{equation*}
   -\ln\mathrm{Vol}(\mathcal{A}) \le\H^{\pi_k}(s)\le 
\frac{R_{\max}-R_{\min}}{\tau(1-\gamma)}
+\frac{\mathrm{Vol}(\mathcal{A})-C_0}{1-\gamma}
\end{equation*}

\begin{lemma}\label{lem:positive}
For any $k,s$, $\pi_{k+1}(\cdot\mid s)$ of \eqref{eq:refwppg2}
satisfies $\pi_{k+1}(\cdot\mid s)\ll p_0$ and $d\pi_{k+1}(\cdot\mid s)/dp_0>0$ $p_0$-a.e., where $p_0$ is uniform distribution.
\end{lemma}

\begin{proof}
From lemma\ref{lem:H bound}
Finite objective value forces $\KL(\pi_{k+1}(\cdot|s)\|p_0)=\H^{\pi_{k+1}}(s)+\ln \mathrm{Vol}(\mathcal{A}) <\infty$, hence $\pi_{k+1}\ll p_0$.
If $B\subset\cA$ has $p_0(B)>0$ and $\pi_{k+1}(B)=0$, consider the perturbation
$q_\varepsilon=(1-\varepsilon)\pi_{k+1}+\varepsilon\,p_0(\cdot\mid B)$.
The linear term $\ip{Q_k}{q_\varepsilon}$ changes by $O(\varepsilon)$ since $Q_k$ is bounded on $\cA$.
Moreover, $\wass^2(q_\varepsilon,\pi_k)\le \wass^2(\pi_{k+1},\pi_k)+\varepsilon D^2$ by convexity of $\wass^2(\cdot,\pi_k)$ under mixtures.
Finally, the entropy gain satisfies
$\KL(q_\varepsilon\|p_0)-\KL(\pi_{k+1}\|p_0)=\varepsilon\log\varepsilon+O(\varepsilon)$ as $\varepsilon\downarrow0$.
Thus the objective in \eqref{eq:refwppg2} increases by
$\tau(-\varepsilon\log\varepsilon)-O(\varepsilon)>0$ for small $\varepsilon$, contradicting optimality.
\end{proof}

\begin{lemma}[Oscillation bound for Kantorovich potentials]\label{lem:oscphi_polished}
Let $\cA\subset\R^d$ be compact with diameter $R_A$. If $\varphi$ is a Kantorovich potential for the quadratic cost
$c(a,a')=\tfrac12\|a-a'\|^2$ between measures supported on $\cA$, then
\[
\osc(\varphi) \le R_A^2.
\]
\end{lemma}

\begin{proof}
This is standard, but we include a short proof for completeness (see also \citep[Ch.~5]{ambrosio2005gradient}, \citep[Ch.~1]{santambrogio2015optimal}).
Choose $c$-conjugate optimal potentials so that
\[
\varphi(a)=\inf_{a'\in\cA}\Big\{\tfrac12\|a-a'\|^2-\psi(a')\Big\}.
\]
For fixed $a'$, set $g_{a'}(a):=\tfrac12\|a-a'\|^2-\psi(a')$. For $a,b\in\cA$,
\[
g_{a'}(a)-g_{a'}(b)=\tfrac12\langle a-b,\ a+b-2a'\rangle,
\]
hence
\[
|g_{a'}(a)-g_{a'}(b)|
\le \tfrac12\|a-b\|(\|a-a'\|+\|b-a'\|)
\le R_A\|a-b\|.
\]
Thus each $g_{a'}$ is $D$-Lipschitz, and so is $\varphi=\inf_{a'}g_{a'}$.
Since $\|a-b\|\le R_A$ on $\cA$, we obtain $\osc(\varphi)\le R_A^2$.
\end{proof}

\begin{lemma}[Bounded perturbations preserve $T_2$]\label{lem:boundedperturbT2}
Let $\mu$ satisfy $T_2(\lambda_0)$ and let $d\nu=\frac{1}{Z}e^{f}\,d\mu$ with $f$ bounded.
Then $\nu$ satisfies $T_2(\lambda)$ with
\[
\lambda\ \ge\ \frac{\lambda_0}{8}\,e^{-\osc(f)}.
\]
\end{lemma}

\begin{proof}
In the notation of \cite{GozlanRobertoSamson2011}, $T_2(\lambda_0)$ is equivalent to $(T_2(C))$ with $C=1/\lambda_0$.
Corollary~1.7 in \cite{GozlanRobertoSamson2011} gives $(T_2(8C e^{\osc(f)}))$ for $\nu$.
Translating back yields $\lambda\ge (\lambda_0/8)e^{-\osc(f)}$.
\end{proof}

\begin{theorem}[Uniform $T_2$]\label{thm:uniformT2}
Under Assumption~\ref{as:Bounded} and \ref{as:entropy bounded}, there exists $\lambda>0$ such that for all $k\ge 1$ and all states $s$,
$\pi_k(\cdot\mid s)$ satisfies $T_2(\lambda)$:
\[
\KL(\nu\|\pi_k(\cdot\mid s))\ \ge\ \frac{\lambda}{2}\W^2\big(\nu,\pi_k(\cdot\mid s)\big)
\qquad\forall \nu\ll \pi_k(\cdot\mid s).
\]
One admissible explicit constant is
\begin{equation}\label{eq:lambda-final_polished}
\lambda
\ :=\
\frac{\lambda_0}{8}\exp\!\Bigg(
-\frac{1}{\tau}\Delta_Q\ -\ \frac{D^2}{\eta\tau}
\Bigg),
\qquad
\Delta_Q=\frac{R_{\max}-R_{\min}+\gamma \tau\mathrm{Vol}(\mathcal{A})-\gamma\tau C_0}{1-\gamma}.
\end{equation}
\end{theorem}
Where $\lambda_0$ is the $T_2$ constant for uniform distribution on the corresponding compact action spaces like ball and hypercube.
\begin{proof} 
Fix $k\ge 0$ and a state $s$. We show that $\pi_{k+1}(\cdot\mid s)$ satisfies $T_2(\lambda)$
with a constant independent of $k$ and $s$.

Rewrite \eqref{eq:refwppg2} as minimization of
\[
\Phi(q):=\tau\H^q(s)-\ip{Q_k(s,\cdot)}{q}+\frac{1}{2\eta_k}\W^2\big(q,\pi_k(\cdot\mid s)\big).
\]
By Lemma~\ref{lem:positive}, $\pi_{k+1}(\cdot\mid s)$ has strictly positive density w.r.t.\ $p_0$, where $p_0$ is uniform distribution on action space, so
$f_{k+1}:=\log\frac{d\pi_{k+1}(\cdot\mid s)}{dp_0}$ is well-defined $p_0$-a.e.
Let $\varphi_{k+1}$ be a Kantorovich potential for the quadratic cost transporting
$\pi_{k+1}(\cdot\mid s)$ to $\pi_k(\cdot\mid s)$.
Standard first-order optimality conditions for Wasserstein proximal maps
(e.g.\ \citep[Ch.~10]{santambrogio2015optimal} or \citep[Ch.~7]{santambrogio2015optimal})
yield that there exists a constant $c_{k+1}\in\R$ such that, $p_0$-a.e.\ on $\cA$,
\begin{equation}\label{eq:logdensity_polished}
f_{k+1}(a)
=
\frac{1}{\tau}Q_k(s,a)\ -\ \frac{1}{\eta_k\tau}\varphi_{k+1}(a)\ -\ c_{k+1}.
\end{equation}
Taking oscillations in \eqref{eq:logdensity_polished} and using that oscillation is invariant under constants,
\[
\osc(f_{k+1})
\le
\frac{1}{\tau}\osc\big(Q_k(s,\cdot)\big)
+\frac{1}{\eta_k\tau}\osc(\varphi_{k+1}).
\]
By Lemma~\ref{lem:Qbounds}, $\osc(Q_k(s,\cdot))\le \Delta_Q$.
By Lemma~\ref{lem:oscphi_polished}, $\osc(\varphi_{k+1})\le D^2$.
This gives the uniform bound
\[
\osc(f_{k+1})
\le
\frac{\Delta_Q}{\tau}+\frac{D^2}{\eta\tau}
=:B.
\]

Since $p_0$ is the uniform distribution on a compact action set $\mathcal A$—in particular, we focus on either a Euclidean ball or a hypercube, both common in practice—$p_0$ satisfies the transportation-information inequality $T_2(\lambda_0)$. 
In what follows, we specialize to these two representative bounded action sets.

(i) If $\mathcal A = B_2(0,R_{ A})\subset\mathbb R^d$ and $\mu=\mathrm{Unif}(A)$, then $\mu$ satisfies $T_2(\lambda_{\mathrm{ball}})$ with
\[
\lambda_{\mathrm{ball}}=\frac{\Gamma(\frac d2+1)^{2/d}}{2R_{ A}^2},
\]
see \citet{bobkov2000brunn} together with the implication $\mathrm{LSI}\Rightarrow T_2$ in \citet{otto2000generalization}.

(ii) If $\mathcal A=[-R_{ A},R_{ A}]^d$ and $\mu=\mathrm{Unif}( A)$, then $\mu$ satisfies $T_2(\lambda_{\mathrm{cube}})$ with
\[
\lambda_{\mathrm{cube}}=\frac{\pi^2}{4R_{A}^2},
\]
see \citet{compris2002remarks} and again \citet{otto2000generalization}.

Since $\pi_{k+1}(\cdot\mid s)$ is an exponential tilt of $p_0$, namely
$d\pi_{k+1} \propto e^{f_{k+1}}\,dp_0$, Lemma~\ref{lem:boundedperturbT2} yields
\[
\pi_{k+1}(\cdot\mid s)\ \text{satisfies}\ T_2(\lambda)\quad
\text{with}\quad
\lambda\ \ge\ \frac{\lambda_0}{8}e^{-\osc(f_{k+1})}\ \ge\ \frac{\lambda_0}{8}e^{-B}.
\]
Substituting the expression for $B$ gives \eqref{eq:lambda-final_polished}.
\end{proof}

\subsection{Proof of the Main Theorems} \label{subsection:Cproof theorem}

\subsubsection{Technical Lemmas}
We first present the performance differance lemma for entropy regularized reinforcement learning.

\begin{lemma}[Entropy Regularized Performance Difference Lemma] \citep[Lemma2]{lan2023policy} \label{lem:Performance Dif ference lemma}
For any two feasible policies $\pi$ and $\pi'$, we have
\[
V_\tau^{\pi'}(s)-V_\tau^{\pi}(s)
=\frac{1}{1-\gamma}
\E_{s'\sim d^{\pi'}_{s}}\Big[
\big\langle A_\tau^{\pi}(s',\cdot),\,\pi'(\cdot|s')\big\rangle
- \tau \H^{\pi'}(s') + \tau \H^{\pi}(s')
\Big],
\]
where $A_\tau^{\pi}(s',a)\;:=\;Q_\tau^{\pi}(s',a)-V_\tau^{\pi}(s')$.
\end{lemma}

For completeness, we provide the proof here and adapt it to our notation.

\begin{proof}
For simplicity, let us denote $\xi^{\pi'}(s_0)$ the random process $(s_t,a_t,s_{t+1}),~t\ge0$, generated by following the policy $\pi'$ starting with the initial state $s_0$. It then follows from the definition of $V_\tau^{\pi'}$ that
\begin{align*}
&V_\tau^{\pi'}(s) - V_\tau^\pi(s) \\
&= \mathbb{E}_{\xi^{\pi'}(s)} \left[ \sum_{t=0}^\infty \gamma^t \big( r(s_t,a_t) - \tau \H^{\pi'}(s_t) \big) \right] - V_\tau^\pi(s) \\
&= \mathbb{E}_{\xi^{\pi'}(s)} \left[ \sum_{t=0}^\infty \gamma^t \big( r(s_t,a_t) - \tau \H^{\pi'}(s_t) + V_\tau^\pi(s_t) - V_\tau^\pi(s_t) \big) \right] - V_\tau^\pi(s) \\
&\overset{(1)}{=} \mathbb{E}_{\xi^{\pi'}(s)} \left[ \sum_{t=0}^\infty \gamma^t \big( r(s_t,a_t) - \tau \H^{\pi'}(s_t) + \gamma V_\tau^\pi(s_{t+1}) - V_\tau^\pi(s_t) \big) \right] \\
&\quad + \mathbb{E}_{\xi^{\pi'}(s)} \left[ V_\tau^\pi(s_0) \right] - V_\tau^\pi(s) \\
&\overset{(2)}{=} \mathbb{E}_{\xi^{\pi'}(s)} \left[ \sum_{t=0}^\infty \gamma^t \big( r(s_t,a_t) - \tau \H^{\pi'}(s_t) + \gamma V_\tau^\pi(s_{t+1}) - V_\tau^\pi(s_t) \big) \right] \\
&= \mathbb{E}_{\xi^{\pi'}(s)} \left[ \sum_{t=0}^\infty \gamma^t \big( r(s_t,a_t) - \tau \H^{\pi}(s_t) + \gamma V_\tau^\pi(s_{t+1}) - V_\tau^\pi(s_t) - \tau \H^{\pi'}(s_t) + \tau \H^\pi(s_t) \big) \right] \\
&\overset{(3)}{=} \mathbb{E}_{\xi^{\pi'}(s)} \left[ \sum_{t=0}^\infty \gamma^t \big( Q_\tau^\pi(s_t,a_t) - V_\tau^\pi(s_t) - \tau \H^{\pi'}(s_t) + \tau \H^\pi(s_t) \big) \right].
\end{align*}

where (1) follows by taking the term $V_\tau^\pi(s_0)$ outside the summation, (2) follows from the fact that $\mathbb{E}_{\xi^{\pi'}(s)}[V_\tau^\pi(s_0)] = V_\tau^\pi(s)$ since the random process starts with $s_0 = s$, and (3) follows from \ref{eq: soft value}. The previous conclusion then imply that
\begin{align*}
&V_\tau^{\pi'}(s) - V_\tau^\pi(s) \\
&= \frac{1}{1-\gamma} \sum_{s'\in\mathcal{S}} \sum_{a'\in\mathcal{A}} d^{\pi'}_\gamma(s') \pi'(a'|s') \big[ A_\tau^\pi(s',a') + \tau \H^{\pi'}(s') - \tau \H^\pi(s') \big] \\
&= \frac{1}{1-\gamma} \sum_{s'\in\mathcal{S}} d^{\pi'}_\gamma(s') \big[ A_\tau^\pi(s',\pi'(\cdot|s')) - \tau \H^{\pi'}(s') + \tau \H^\pi(s') \big],
\end{align*}
which immediately implies the result.
\end{proof}

Remember our defination of $\nu^\ast$ as the steady state distribution induced by $\pi^*$. 

\begin{lemma}\citep[Lemma3]{lan2023policy}\label{VQtransformation}
\begin{align}
\mathbb{E}_{s\sim\nu^*} \left[ Q_\tau^\pi(s,\cdot),  \pi^*(\cdot|s)-\pi(\cdot|s) - \tau \H^{\pi^{\ast}}(s) + \tau \H^{\pi}(s) \right] 
&= \mathbb{E}_{s\sim\nu^*} \left[ (1-\gamma) \left( V_\tau^{\pi^\ast}(s) - V_\tau^{\pi}(s) \right) \right]. 
\end{align}
\end{lemma}

For completeness, we provide the proof here and adapt it to our notation.

\begin{proof}
It follows from Lemma \ref{lem:Performance Dif ference lemma} (with $\pi' = \pi^*$) that
\[
(1-\gamma) \left[ V_\tau^{\pi^*}(s) - V_\tau^\pi(s) \right] 
= \mathbb{E}_{s'\sim d^{\pi^\ast}_s} \left[ A_\tau^\pi(s',\cdot), \pi^*(\cdot|s') + \tau \H^\pi(s') - \tau \H^{\pi^*}(s') \right].
\]
Noting that:
\begin{align}
\langle A_\tau^\pi(s',\cdot), \pi^*(\cdot|s') \rangle 
&= \langle Q_\tau^\pi(s',\cdot), \pi^*(\cdot|s') \rangle - V_\tau^\pi(s') \\
&= \langle Q_\tau^\pi(s',\cdot), \pi^*(\cdot|s') \rangle - \langle Q_\tau^\pi(s',\cdot), \pi(\cdot|s') \rangle \\
&= \langle Q_\tau^\pi(s',\cdot), \pi^*(\cdot|s') - \pi(\cdot|s') \rangle,
\end{align}
Combining the above two relations and taking expectation w.r.t. $\nu^*$, we obtain
\begin{align*}
(1-\gamma) \mathbb{E}_{s\sim\nu^*} \left[ V_\tau^{\pi^*}(s) - V_\tau^\pi(s) \right]
&= \mathbb{E}_{s\sim\nu^*,\, s'\sim d^{\pi^*}_s} \left[ \langle Q_\tau^\pi(s',\cdot), \pi^*(\cdot|s') - \pi(\cdot|s') \rangle + \tau \H^\pi(s') - \tau \H^{\pi^*}(s') \right] \\
&= \mathbb{E}_{s\sim\nu^*} \left[ Q_\tau^\pi(s,\cdot), \pi^*(\cdot|s) - \pi(\cdot|s) - \tau \H^{\pi^*}(s) \right]+ \tau \H^\pi(s)
\end{align*}
where the second identity is due to $\nu^*$ is the steady state distribution induced by $\pi^*$. 
\end{proof}

Next we will  show a  geometry property of the squared $\wass_2$ distance that we repeatedly leverage in our convergence analysis. 

\begin{lemma}
\label{200}
Let \(\nu \in \cP_2(\R^d)\).
For any \(\rho,\mu \in \cP_2(\R^d)\), let
\((\varphi^{\rho\to\nu}, \psi^{\rho\to\nu})\) be an optimal Kantorovich dual pair
for the cost \(c(x,y)=\frac{1}{2}\|x-y\|^2\) between \((\rho,\nu)\), i.e.
\begin{equation}
\varphi(x)+\psi(y) \le \frac{1}{2}\|x-y\|^2 \quad \text{for all } x,y \in \R^d
\label{eq:dual-constraint}
\end{equation}
\begin{equation}
\int \varphi^{\rho\to\nu}  d\rho + \int \psi^{\rho\to\nu}d\nu= \frac{1}{2} \wass_2^2(\rho,\nu).
\label{eq:dual-attain}
\end{equation}
Then, for every \(\mu \in \cP_2(\R^d)\),
\begin{equation}\label{eq:W2-3point}
\frac{1}{2} \wass_2^2(\mu,\nu)
\ge
\frac{1}{2} \wass_2^2(\rho,\nu)
+
\int_{\R^d} \varphi^{\rho\to\nu}(x)\big(\mu-\rho\big)(dx).
\end{equation}
\end{lemma}

\begin{proof}\label{W2distanceproof}
Recall the dual formulation (valid for any \(\mu \in \cP_2(\R^d)\)):
\begin{equation}\label{eq:dual-form}
\frac{1}{2} \wass_2^2(\mu,\nu)
= \sup_{\varphi,\psi}
\Big\{
  \int \varphi \, d\mu + \int \psi \, d\nu
  :
  \varphi(x)+\psi(y) \le \frac{1}{2}\|x-y\|^2
\Big\}.
\end{equation}
Since the feasibility constraint \eqref{eq:dual-constraint} is pointwise in \((x,y)\),
the optimal pair \((\varphi^{\rho\to\nu}, \psi^{\rho\to\nu})\) for \((\rho,\nu)\) is also a
feasible pair in the supremum \eqref{eq:dual-form} for \((\mu,\nu)\).
Therefore,
\begin{equation}\label{eq:plug-feasible}
\frac{1}{2} \wass_2^2(\mu,\nu)
\ge
\int \varphi^{\rho\to\nu} \, d\mu + \int \psi^{\rho\to\nu} \, d\nu.
\end{equation}
By optimality for \((\rho,\nu)\), we have \eqref{eq:dual-attain}. Subtracting
\(\int \varphi^{\rho\to\nu} \, d\rho\) on the right-hand side of \eqref{eq:plug-feasible} yields
\begin{align*}
\frac{1}{2} \wass_2^2(\mu,\nu)
&\ge
\Big( \int \varphi^{\rho\to\nu} \, d\rho + \int \psi^{\rho\to\nu} \, d\nu \Big)
+ \int \varphi^{\rho\to\nu} \, d(\mu-\rho)
\\&=
\frac{1}{2} \wass_2^2(\rho,\nu)
+ \int \varphi^{\rho\to\nu} \, d(\mu-\rho),
\end{align*}
which is exactly \eqref{eq:W2-3point}.
\end{proof}

\subsubsection{Exact value function case}\label{app:exact}

\begin{lemma}[Wasserstein proximal one step inequality]\label{lem:w2-proximal-step}
Under assumption \ref{as:T2}, fix a state $s\in\cS$ and let the per-state proximal update be
\begin{equation}\label{eq:w2-prox-step}
\pi_{k+1}(\cdot|s)\ \in\ \argmax_{q\in\cP_{\cA}}
\Big\{ \langle Q^{\pi_k}(s,\cdot),q\rangle -\tau \H^{q}(s) - \frac{1}{2\eta_k}\wass_2^2\!\big(q,\pi_k(\cdot|s)\big)\Big\},
\end{equation} 
where $\langle f,q\rangle:=\int_{a\in\cA} f(a)\,q(a)\d a$.
Then for any competitor $p\in\Delta_{\cA}$,
\begin{align*}\label{eq:w2-lemma4}
&\eta_k \Big(\big\langle Q^{\pi_k}(s,\cdot),p-\pi_{k+1}(\cdot|s)\big\rangle -\tau  \H^{p}(s)+\tau \H^{\pi_{k+1}}(s)  \Big)
+\tfrac12\wass_2^2 \big(\pi_{k+1}(\cdot|s),\pi_k(\cdot|s)\big) \\
&\le
\tfrac12\wass_2^2 \big(p,\pi_k(\cdot|s)\big)
- \tfrac{\eta_k\lambda\tau}{2} \wass_2^2 \big(p,\pi_{k+1}(\cdot|s)\big).
\end{align*}
\end{lemma}

\begin{proof}
Let $\varphi^{\pi_{k+1}\to\pi_k}(s,\cdot)$ be a Kantorovich potential for the pair
$\big(\bar\pi_{k+1}(\cdot|s),\,\pi_k(\cdot|s)\big)$ under cost $c(a,a')=\tfrac12\|a-a'\|^2$.
Note that $\langle Q^{\pi_k}(s,\cdot),q\rangle$ is a linear functional of $q$, 
$-\tau \H^{q}(s)$ is strongly concave in $q$, 
and $-\tfrac{1}{2\eta_k}\wass_2^2\!\big(q,\pi_k(\cdot|s)\big)$ is concave in $q$.
The first-order optimality condition of the concave program \eqref{eq:w2-prox-step} states that
\begin{equation}
\Big\langle \eta_k \big(Q^{\pi_k}(s,\cdot)- \tau(1+\ln\pi_{k+1}(\cdot|s)) \big) -  \varphi^{\pi_{k+1}\to\pi_k}(\cdot,a),\; p(\cdot|s)-\pi_{k+1}(\cdot|s)\Big\rangle \le 0,
\qquad \forall\,p(\cdot|s)\in\cP_{\cA}.
\end{equation}

Next, apply ~\ref{200} for $F(p)=\tfrac12\,\wass_2^2(p,\pi_k(\cdot|s))$
and arbitrary $p(\cdot|s)$:
\begin{equation}\label{eq:w2}
\frac12\,\wass_2^2\big(p(\cdot|s),\pi_k(\cdot|s)\big)
\ \ge\
\frac12\,\wass_2^2\!\big(\pi_{k+1}(\cdot|s),\pi_k(\cdot|s)\big)
\ +\ \big\langle \varphi^{\pi_{k+1}\to\pi_k}(s,\cdot),\,p(\cdot|s)-\pi_{k+1}(\cdot|s)\big\rangle.
\end{equation}
Rearranging \eqref{eq:w2} gives
\[
\big\langle \varphi^{\pi_{k+1}\to\pi_k},p-\pi_{k+1}\big\rangle
\ \le\ \frac12\wass_2^2(p,\pi_k)\;-\frac12\wass_2^2(\pi_{k+1},\pi_k),
\]
where the arguments $(\cdot|s)$ are omitted for readability. Plug this bound into
optimal condition to obtain 
\begin{align}
\Big\langle \eta_k \big(Q^{\pi_k}(s,\cdot)- \tau(1+\ln\pi_{k+1}) \big) ,\; p-\pi_{k+1} \Big\rangle &+ \frac12\wass_2^2(\pi_{k+1},\pi_k) \\&\le \frac12\wass_2^2(p,\pi_k),
\qquad \forall\,p(\cdot|s)\ in\cP_{\cA}.
\end{align}

By noting the two facts that 
\[\langle 1,p-\pi_{k+1}(\cdot|s) \rangle =0,\]
and 
\begin{align}
\langle \ln\pi_{k+1}(\cdot|s),p \rangle
&= \langle \ln\pi_{k+1}(\cdot|s),p \rangle -\langle \ln p,p \rangle + \langle \ln p,p \rangle \\
&=- \KL(p\|\pi_{k+1}(\cdot|s)) + \H^p(s),
\end{align}
By assumption\ref{as:T2} and combine all the above, we have 
 \begin{equation}
\begin{aligned}
&\eta_k \Big(\big\langle Q^{\pi_k}(s,\cdot),p-\pi_{k+1}(\cdot|s)\big\rangle -\tau  \H^{p}(s)+\tau \H^{\pi_{k+1}}(s)  \Big)
+\frac12\wass_2^2 \big(\pi_{k+1}(\cdot|s),\pi_k(\cdot|s)\big) \\
&\le
\frac12\wass_2^2 \big(p,\pi_k(\cdot|s)\big)
- \frac{\eta_k\lambda\tau}{2} \wass_2^2 \big(p,\pi_{k+1}(\cdot|s)\big).
\end{aligned}
 \end{equation}
\end{proof}

\begin{lemma}
For any $s \in \cS$, we have
\begin{align}
V_\tau^{\pi_{k+1}}(s) - V_\tau^{\pi_k}(s)
&\ge \langle Q_\tau^{\pi_k}(s,\cdot), \pi_{k+1}(\cdot|s) - \pi_k(\cdot|s) \rangle 
- \tau \H^{\pi_{k+1}}(s) + \tau \H^{\pi_k}(s). \label{20}
\end{align}
\end{lemma}

\begin{proof}
It follows from Lemma \ref{lem:Performance Dif ference lemma} (with $\pi' = \pi^{k+1}$, $\pi = \pi^k$) that
\begin{align}
V_\tau^{\pi_{k+1}}(s) - V_\tau^{\pi_k}(s) 
&= \frac{1}{1-\gamma} \E_{s'\sim d^{\pi_k}_s} \Big[ 
\langle A_\tau^{\pi_k}(s',\cdot), \pi_{k+1}(\cdot|s') \rangle 
- \tau \H^{\pi_{k+1}}(s') + \tau \H^{\pi_k}(s') \Big].
\end{align}
And
\begin{align*}
\langle A_\tau^{\pi_k}(s',\cdot), \pi_{k+1}(\cdot|s') \rangle 
&= \langle Q_\tau^{\pi_k}(s',\cdot), \pi_{k+1}(\cdot|s') \rangle - V_\tau^{\pi_k}(s') \\
&= \langle Q_\tau^{\pi_k}(s',\cdot), \pi_{k+1}(\cdot|s') \rangle 
- \langle Q_\tau^{\pi_k}(s',\cdot), \pi_k(\cdot|s') \rangle \\
&= \langle Q_\tau^{\pi_k}(s',\cdot), \pi_{k+1}(\cdot|s') - \pi_k(\cdot|s') \rangle.
\end{align*}
Combining the two identities above, we obtain
\begin{align}\label{22}
V_\tau^{\pi_{k+1}}(s) - V_\tau^{\pi_k}(s) 
= \frac{1}{1-\gamma} \E_{s'\sim d^{\pi_{k+1}}_s} \Big[ 
\langle Q_\tau^{\pi_k}(s',\cdot), \pi_{k+1}(\cdot|s') - \pi_k(\cdot|s') \rangle 
- \tau \H^{\pi_{k+1}}(s') + \tau \H^{\pi_k}(s') \Big]. 
\end{align}

Now we conclude from Lemma \ref{lem:w2-proximal-step} with $p = \pi_k(\cdot|s')$ for any $s'$ that
\begin{align}\label{23}
\langle Q_\tau^{\pi_k}(s',\cdot), \pi_{k+1}(\cdot|s') - \pi_k(\cdot|s') \rangle 
- \tau \H^{\pi_{k+1}}(s') + \tau \H^{\pi_k}(s') 
\;\ge\; \frac{\eta_k\lambda\tau}{2}\wass_2^2\big(\pi_k(\cdot|s'),\pi_{k+1}(\cdot|s')\big). 
\end{align}
The previous two conclusions then clearly imply the result in \eqref{20}.

For any fixed $s'\in\cS$, since $\pi_{k+1}(\cdot|s')$ maximizes the per-state WPPG objective~\eqref{eq:w2-prox-step},
we can compare it with the feasible competitor $q=\pi_k(\cdot|s')$ to obtain
\begin{align*}
&\Big\langle Q_\tau^{\pi_k}(s',\cdot),\,\pi_{k+1}(\cdot|s')\Big\rangle
-\tau \H^{\pi_{k+1}}(s')
-\frac{1}{2\eta_k}\W^2\!\Big(\pi_{k+1}(\cdot|s'),\,\pi_k(\cdot|s')\Big) \\
&\qquad\ge~
\Big\langle Q_\tau^{\pi_k}(s',\cdot),\,\pi_k(\cdot|s')\Big\rangle
-\tau \H^{\pi_k}(s')
-\frac{1}{2\eta_k}\W^2\!\Big(\pi_k(\cdot|s'),\,\pi_k(\cdot|s')\Big).
\end{align*}
Since $\W^2(\pi_k(\cdot|s'),\pi_k(\cdot|s'))=0$, rearranging yields
\begin{align}
\Big\langle Q_\tau^{\pi_k}(s',\cdot),\,\pi_{k+1}(\cdot|s')-\pi_k(\cdot|s')\Big\rangle
-\tau \H^{\pi_{k+1}}(s')+\tau \H^{\pi_k}(s')
~\ge~
\frac{1}{2\eta_k}\W^2\!\Big(\pi_{k+1}(\cdot|s'),\,\pi_k(\cdot|s')\Big)
~\ge~0,
\label{eq:integrand_nonneg}
\end{align}
which holds for all $s'\in\cS$.

Then combine with \eqref{23}, we have that
\begin{align*}
&\E_{s'\sim d^{\pi_{k+1}}_{s(s)}} \Big[ 
\langle Q_\tau^{\pi_k}(s',\cdot), \pi_{k+1}(\cdot|s') - \pi_k(\cdot|s') \rangle 
- \tau \H^{\pi_{k+1}}(s') + \tau \H^{\pi_k}(s') \Big] \\
&\;\ge\; d^{\pi_{k+1}}_{s(s)}\, \Big[ 
\langle Q_\tau^{\pi_k}(s,\cdot), \pi_{k+1}(\cdot|s) -  \pi_k(\cdot|s) \rangle 
- \tau \H^{\pi_{k+1}}(s) + \tau \H^{\pi_k}(s) \Big] \\
&\;\ge\; (1-\gamma) \Big[ 
\langle Q_\tau^{\pi_k}(s,\cdot), \pi_{k+1}(\cdot|s) -  \pi_k(\cdot|s) \rangle 
- \tau \H^{\pi_{k+1}}(s) + \tau \H^{\pi_k}(s) \Big],
\end{align*}
where the last inequality follows from the fact that $d^{\pi_k}_s(s) \ge (1-\gamma)$ due to the definition of $d^{\pi_k}_s$ and $s_0 = s$ with probability one. 
Then by \eqref{22} and the above inequality, the claim follows.
\end{proof}

\begin{theorem}\label{app:thm2} 
Suppose Assumption \ref{as:T2} holds and set the step size $\eta_k=\eta =\frac{1}{\gamma\lambda\tau}$. 
Then for any $k \ge 0$, the iterates of \eqref{eq:WPPG} satisfy
\begin{align*}
&J(\pi_{\ast}) - J(\pi_{k}) + \lambda\tau \mathcal{D}(\pi_{k}, \pi^\ast) \\ &\leq \gamma^k \left[ J(\pi^\ast) - J(\pi_{0}) + \lambda \tau \mathcal{D}(\pi_{0}, \pi^\ast)\right] 
\end{align*}
where $J$ is defined in \eqref{defination of J}, and
\begin{equation*}
\mathcal{D}(\pi_k, \pi^\ast) := 
\mathbb{E}_{s \sim \nu^\ast} \left[ \tfrac{1}{2}\,\wass_2^2\bigl(\pi_k(\cdot|s), \pi^\ast(\cdot|s)\bigr) \right].
\end{equation*}

Consequently, in order to achieve an error of $\mathcal{O}(\varepsilon+\delta)$, 
the required iteration complexity is
\begin{align*}
\mathcal{O}\!\left(\frac{1}{1-\gamma}\,\log \frac{\,J(\pi^\ast) - J(\pi _0) + \lambda \tau \mathcal{D}(\pi_0, \pi^\ast)}{\varepsilon}\right).
\end{align*}
\end{theorem}

\begin{proof}
By Lemma \ref{lem:w2-proximal-step} with $p = \pi^\ast$, we have
\begin{align*}\label{eq:w2-lemma4}
&\ \big\langle Q_\tau^{\pi_k}(s,\cdot), \pi^\ast(\cdot|s) - \pi_{k+1}(\cdot|s) \big\rangle 
   - \tau \H^{\pi_\ast}(s) + \tau \H^{\pi_{k+1}}(s)  
   + \frac{1}{2\eta_k}\wass_2^2\big(\pi_{k+1}(\cdot|s), \pi_k(\cdot|s)\big) \\
&\le \frac{1}{2\eta_k}\wass_2^2\big(\pi_k(\cdot|s), \pi^\ast(\cdot|s)\big)
   - \frac{\lambda\tau}{2}\wass_2^2\big(\pi_{k+1}(\cdot|s), \pi^\ast(\cdot|s)\big).
\end{align*}

Combining with \eqref{20}, we obtain 
\begin{align*}
&\Bigl[
    \langle Q_\tau^{\pi_k}(s,\cdot), \pi^\ast(\cdot|s) - \pi_k(\cdot|s) \rangle
    - \tau \H^{\pi^\ast}(s) + \tau \H^{\pi_k}(s)
\Bigr] \\
&\quad + \Bigl[ V_\tau^{\pi_k}(s) - V_\tau^{\pi_{k+1}}(s) \Bigr] \\
&\quad + \frac{1}{2\eta_k}\wass_2^2\big(\pi_k(\cdot|s),  \pi_{k+1}(\cdot|s)\big) \\[6pt]
&\le \Bigl[
    \langle Q_\tau^{\pi_k}(s,\cdot), \pi^\ast(\cdot|s) - \pi_k(\cdot|s) \rangle
    - \tau \H^{\pi^\ast}(s) + \tau \H^{\pi_k}(s)
\Bigr] \\
&\quad - \Bigl[
    \langle Q_\tau^{\pi_k}(s,\cdot), \pi_{k+1}(\cdot|s) - \pi_k(\cdot|s) \rangle
    - \tau \H^{\pi_{k+1}}(s) + \tau \H^{\pi_k}(s)
\Bigr] \\
&\quad + \frac{1}{2\eta_k}\wass_2^2\big(\pi_k(\cdot|s), \pi_{k+1}(\cdot|s)\big) \\[6pt]
&= \Bigl(
    \langle Q_\tau^{\pi_k}(s,\cdot), \pi^\ast(\cdot|s) - \pi_{k+1}(\cdot|s) \rangle
    - \tau \H^{\pi^\ast}(s) + \tau \H^{\pi_{k+1}}(s)
\Bigr) \\
&\quad + \frac{1}{2\eta_k}\wass_2^2\big(\pi_{k+1}(\cdot|s), \pi_k(\cdot|s)\big) \\[6pt]
&\le \frac{1}{2\eta_k}\wass_2^2\big(\pi_k(\cdot|s), \pi^\ast(\cdot|s)\big)
   - \frac{\lambda\tau}{2}\wass_2^2\big(\pi_{k+1}(\cdot|s), \pi^\ast(\cdot|s)\big).
\end{align*}

Taking expectation with respect to $\nu^\ast$ on both sides of the inequality, we obtain
\begin{align*}
\mathbb{E}_{s \sim \nu^\ast} \Bigl[
    (1-\gamma)\bigl(V_\tau^{\pi^\ast}(s) - V_\tau^{\pi_k}(s)\bigr)
\Bigr]
&+ \mathbb{E}_{s \sim \nu^\ast} \Bigl[
    V_\tau^{\pi_k}(s) - V_\tau^{\pi_{k+1}}(s)
\Bigr] \\
&+ \mathbb{E}_{s \sim \nu^\ast} \Bigl[
    \frac{1}{2\eta_k}\wass_2^2\big(\pi_k(\cdot|s), \pi_{k+1}(\cdot|s)\big)
\Bigr] \\[6pt]
&\le \mathbb{E}_{s \sim \nu^\ast} \Bigl[
    \frac{1}{2\eta_k}\wass_2^2\big(\pi_k(\cdot|s), \pi^\ast(\cdot|s)\big)
    - \frac{\lambda\tau}{2}\wass_2^2\big(\pi_{k+1}(\cdot|s), \pi^\ast(\cdot|s)\big)
\Bigr].
\end{align*}

By rewriting 
\[
V_\tau^{\pi_k}(s) - V_\tau^{\pi_{k+1}}(s)
= V_\tau^{\pi_k}(s) - V_\tau^{\pi^\ast}(s) + V_\tau^{\pi^\ast}(s) - V_\tau^{\pi_{k+1}}(s),
\]
and rearranging the inequality, we have
\begin{align*}
\mathbb{E}_{s \sim \nu^\ast} \Bigl[
    V_\tau^{\pi^\ast}(s) - V_\tau^{\pi_{k+1}}(s)
\Bigr]
&+ \lambda\tau\, \mathbb{E}_{s \sim \nu^\ast} \Bigl[
    \tfrac{1}{2}\wass_2^2\!\left(\pi_{k+1}(\cdot|s), \pi^\ast(\cdot|s)\right)
\Bigr] \\
&+ \mathbb{E}_{s \sim \nu^\ast} \Bigl[
    \tfrac{1}{2}\wass_2^2\!\left(\pi_k(\cdot|s), \pi_{k+1}(\cdot|s)\right)
\Bigr] \\[6pt]
&\le \gamma\, \mathbb{E}_{s \sim \nu^\ast} \Bigl[
    V_\tau^{\pi^\ast}(s) - V_\tau^{\pi_k}(s)
    + \tfrac{1}{2\eta_k\gamma}\wass_2^2\!\left(\pi_k(\cdot|s), \pi^\ast(\cdot|s)\right)
\Bigr]  .
\end{align*}

Thus,
\begin{align*}
&\mathbb{E}_{s \sim \nu^\ast} \Bigl[
    V_\tau^{\pi^\ast}(s) - V_\tau^{\pi_{k+1}}(s)
    + \tfrac{\lambda\tau}{2}\wass_2^2(\pi_{k+1}(\cdot|s), \pi^\ast(\cdot|s))
\Bigr] \\
&\le \gamma\, \mathbb{E}_{s \sim \nu^\ast} \Bigl[
    V_\tau^{\pi^\ast}(s) - V_\tau^{\pi_k}(s)
    + \tfrac{1}{2\eta_k\gamma}\wass_2^2(\pi_k(\cdot|s), \pi^\ast(\cdot|s))
\Bigr] .
\end{align*}

Recalling the definitions of $J$ \eqref{defination of J} and $\mathcal{D}$, we obtain
\begin{equation}
J(\pi^\ast) - J(\pi_{k+1}) + \lambda\tau \mathcal{D}(\pi_{k+1}, \pi^\ast)
\le \gamma \Bigl[ J(\pi^\ast) - J(\pi_k) + \tfrac{1}{\eta_k\gamma}\mathcal{D}(\pi_k, \pi^\ast) \Bigr].
\end{equation}

For any $\eta_k = \eta \ge \tfrac{1}{\gamma \lambda \tau}$ in the JKO scheme, we obtain
\begin{align*}
J(\pi^\ast) - J(\pi_{k+1}) + \lambda\tau \mathcal{D}(\pi_{k+1}, \pi^\ast)
&\le \gamma \Bigl[ J(\pi^\ast) - J(\pi_k) + \lambda\tau \mathcal{D}(\pi_k, \pi^\ast) \Bigr] ,
\end{align*}
which implies
\[
J(\pi^\ast) - J(\pi_k) + \lambda\tau \mathcal{D}(\pi_k, \pi^\ast)
\le \gamma^k \Bigl[ J(\pi^\ast) - J(\pi_0) + \lambda\tau \mathcal{D}(\pi_0, \pi^\ast) \Bigr].
\]
\end{proof}

\subsubsection{Inexact value function case}\label{app:Inexact}

In the next part, we will show how inexact Q affect our results. For the ease of presentation, we denote $\Delta_k = Q^{\pi_k,\xi_k}-Q^{\pi_k}$ and $\xi_{0:k}= \{\xi_0,\xi_1,\cdots,\xi_k \}$ in the following paper.

Recall our added assumption\ref{assumption:stochastic Q}:
For each iteration $k \ge 0$, the stochastic estimator $Q^{\pi_k, \xi_k}$ satisfies
\begin{align*}
\E_{\xi_k}\left[ Q^{\pi_k, \xi_k} \right] &= \bar{Q}^{\pi_k},  \\ \big\| \bar{Q}^{\pi_k} - Q^{\pi_k} \big\|_\infty &\le \epsilon_k,\\\E_{\xi_k}\left[ \big \| \nabla _aQ^{\pi_k, \xi_k} - \nabla_a Q^{\pi_k} \big\|_{2,\infty}^2 \right] &\le \sigma_k^2.        \label{eq:assump3}
\end{align*}

The optimization iteration becomes

\begin{equation}
\label{stochastic wpo update}
\pi_{k+1}(\cdot\mid s)
\in \arg\max_{q(\cdot\mid s)\in \Pi(s)}
\Big\{\langle Q^{\pi_k,\xi_k}(s,\cdot),\, q(\cdot\mid s)\rangle- \tau\, \H^{q}(s) - \frac{1}{2\eta_k}\,
\wass_2^2\!\big(q(\cdot\mid s), \pi_k(\cdot\mid s)\big)
\Big\}
\end{equation}

\begin{lemma}
Under Assumption \ref{assumption:stochastic Q}, for any state s we have: 
\begin{equation}\label{Delta}
\E_{\xi_{0:k}}[\langle \Delta_k(\cdot,s), \pi_{k+1}(\cdot|s)-\pi_k(\cdot|s) \rangle] \leq 2\eta_k \sigma_k^2 + \frac{1}{2\eta_k} \E_{\xi_{0:k}}\wass_2^2(\pi_k(\cdot|s),\pi_{k+1}(\cdot|s))
\end{equation}
\end{lemma}

\begin{proof}

For any s, let $\gamma(a,a'|s)$ be the optimal couple of the two distribution $\pi_k(\cdot|s)$ and $\pi_{k+1}(\cdot|s)$ in $\wass_2$.

\begin{align*}
&\E_{\xi_{0:k}}[\langle \Delta_k(a,s), \pi_{k+1}(a|s)-\pi_k(a|s) \rangle|\xi_{0:k-1}] \\ &=\E_{\xi_k} [\int_\cA \Delta_k(\cdot,s) \mathrm{d}(\pi_{k+1}(\cdot|s)-\pi_{k}(\cdot|s))|\xi_{0:k-1}]\\
&=\E_{\xi_{0:k}} [\iint_{\cA \times \cA} \Delta_k(a,s) - \Delta_k(a',s) \mathrm{d}\gamma(a,a'|s)|\xi_{0:k-1}]\\
&=\E_{\xi_{0:k}} [\iint_{\cA \times \cA} \int\langle \nabla_a \Delta_k ((1-t)a'+ta,s),a-a' \rangle \mathrm{d} t \mathrm{d} \gamma(a,a'|s) |\xi_{0:k-1}]\\
&=\iint_{\cA \times \cA} \int \E_{\xi_{0:k}}[\langle \nabla_a \Delta_k ((1-t)a'+ta,s),a-a' \rangle  |\xi_{0:k-1}] \mathrm{d} t \mathrm{d} \gamma(a,a'|s)\\
&\leq \iint_{\cA \times \cA} \int \E_{\xi_{0:k}}[ 2\eta_k\|\nabla_a \Delta_k ((1-t)a'+ta,s)\|_2^2 + \frac{1}{2\eta_k} \|a-a'\|_2^2 |\xi_{0:k-1}] \mathrm{d} t \mathrm{d}\gamma(a,a'|s)\\ 
&\leq 2\eta_k \sigma_k^2 + \frac{1}{2\eta_k} \E_{\xi_{0:k}}[\wass_2^2(\pi_k(\cdot|s),\pi_{k+1}(\cdot|s))|\xi_{0:k-1}]
\end{align*}

The second equality applies the definition of an optimal coupling 
$\gamma(\cdot,\cdot|s)\in\Gamma(\pi_{k+1}(\cdot|s),\pi_k(\cdot|s))$, 
which means has the same marginal distribution as $\pi_k$ and $\pi_{k+1}$. 
The second inequality uses Young’s inequality 
$\langle u,v\rangle \le \tfrac{1}{2\eta_k}\|u\|^2 + \tfrac{\eta_k}{2}\|v\|^2$ 
to separate the two terms. 
The last inequality bounds the variance term of the stochastic gradient by $\sigma_k^2$ 
yields the last inequality, where the quadratic term recovers 
the squared Wasserstein distance between $\pi_k(\cdot|s)$ and $\pi_{k+1}(\cdot|s)$.

Taking expectation with respect to $\xi_{0:k-1}$ on both sides, we have the final result:
\begin{equation*}
\E_{\xi_{0:k}}[\langle \Delta_k(\cdot,s), \pi_{k+1}(\cdot|s)-\pi_k(\cdot|s) \rangle] \leq 2\eta_k \sigma_k^2 + \frac{1}{2\eta_k} \E_{\xi_{0:k}}\wass_2^2(\pi_k(\cdot|s),\pi_{k+1}(\cdot|s))
\end{equation*}
\end{proof}

\begin{theorem}\label{thm:inexactconvergence_app}\label{app:thm3}
Suppose Assumptions \ref{as:T2} and \ref{assumption:stochastic Q} hold, and   for all $k \ge 0$, $\epsilon_k\le \epsilon, \sigma_k\le\sigma$. Then the iterates of \eqref{stochastic wpo update} using step size $\eta_k=\eta\ge\frac{1}{\gamma\lambda\tau}$ satisfies
\begin{equation}\label{eq:main_ineq_app}
\begin{aligned}
\mathbb{E}_{\xi_{0:k-1}}\!\Big[
&J(\pi^\ast) - J(\pi_k) + \lambda\tau\, D(\pi_k,\pi^\ast)
\Big]
 \\
& \le \gamma^k \Big[
J(\pi^\ast) - J(\pi_0) + \lambda\tau\, D(\pi_0,\pi^\ast)
\Big]  + \mathcal{O}(\epsilon+\sigma).
\end{aligned}
\end{equation}

where $J$ is defined in \eqref{defination of J}, and
$\mathcal{D}(\pi_k, \pi^\ast) := 
\mathbb{E}_{s \sim \nu^\ast} \left[ \tfrac{1}{2}\,\wass_2^2\bigl(\pi_k(\cdot|s), \pi^\ast(\cdot|s)\bigr) \right].
$
Consequently, in order to achieve an error of $\mathcal{O}(\varepsilon+\epsilon+\sigma)$ in expectation, 
the required iteration complexity is
\[
\mathcal{O}\!\left(\frac{1}{1-\gamma}\,\log \frac{\,J(\pi^\ast) - J(\pi_0) + \lambda \tau \mathcal{D}(\pi_0, \pi^\ast)}{\varepsilon}\right).
\]
\end{theorem}

\begin{proof}
By Lemma \ref{lem:w2-proximal-step} applied to \ref{stochastic wpo update} with $p = \pi^\ast$, we have
\begin{equation}\label{eq:w2-lemma4-stochastic}
\begin{aligned}
& \big\langle Q^{\pi_k,\xi_k}_{\tau}(s,\cdot),\pi^\ast(\cdot|s)-{\pi}_{k+1}(\cdot|s)  \big\rangle -\tau  \H^{{\pi^\ast}}(s) +\tau \H^{{\pi}_{k+1}}(s) 
+\frac{1}{2\eta_k}\wass_2^2\!\big({\pi}_{k+1}(\cdot|s),\pi_k(\cdot|s)\big) \\
&\le \frac{1}{2\eta_k}\wass_2^2\big(\pi_k(\cdot|s),\pi^\ast(\cdot|s)\big)
- \frac{\lambda\tau}{2} \wass_2^2 \big({\pi}_{k+1}(\cdot|s),\pi^\ast(\cdot|s) \big).
\end{aligned}    
\end{equation}

By Lemma \ref{lem:w2-proximal-step} applied to \ref{stochastic wpo update} with $p = \pi_k$, we have
\begin{align*}\label{less0}
&\Bigl(\big\langle Q^{\pi_k,\xi_k}_{\tau}(s,\cdot),\pi_k(\cdot|s)-{\pi}_{k+1}(\cdot|s)  \big\rangle -\tau  \H^{{\pi_k}}(s) +\tau \H^{{\pi}_{k+1}}(s)  \Bigr)
+\frac{1}{2\eta_k }\wass_2^2\big({\pi}_{k+1}(\cdot|s),\pi_k(\cdot|s)\big) \\
&\le - \frac{\lambda\tau}{2} \wass_2^2 \big({\pi}_{k+1}(\cdot|s),\pi^\ast(\cdot|s) \big) \le0.
\end{align*}

Which implies that 
\begin{equation}\label{Ebound} 
\begin{aligned}
&\mathbb{E}_{s'\sim d^{\pi_k}_s} \big[ \langle Q^{\pi_k,\xi_k}_{\tau}(s',\cdot), {\pi}_{k+1}(\cdot|s') - \pi_k(\cdot|s') \rangle -\tau \H^{{\pi}_{k+1}}(s') + \tau \H^{\pi_k}(s')+\frac{1}{2\eta_k }\wass_2^2\big({\pi}_{k+1}(\cdot|s),\pi_k(\cdot|s)\big) \big]  \\
&\le d^{\pi_k}_s(s) \big[ \langle Q^{\pi_k}_{\tau}(s,\cdot), {\pi}_{k+1}(\cdot|s) - \pi_k(\cdot|s) \rangle -\tau \H^{{\pi}_{k+1}}(s) + \tau \H^{\pi_k}(s)+\frac{1}{2\eta_k }\wass_2^2\big({\pi}_{k+1}(\cdot|s),\pi_k(\cdot|s)\big) \big]  \\
&\le (1-\gamma) \big[ \langle Q^{\pi_k}_{\tau}(s,\cdot), {\pi}_{k+1}(\cdot|s) - \pi_k(\cdot|s) \rangle - \tau \H^{{\pi}_{k+1}}(s) + \tau \H^{\pi_k}(s) +\frac{1}{2\eta_k }\wass_2^2\big({\pi}_{k+1}(\cdot|s),\pi_k(\cdot|s)\big)\big]
\end{aligned}
\end{equation}

where the last inequality follows from the fact that $d^{\pi_k}_s(s) \ge (1-\gamma)$ due to the definition of $d^{\pi_{k}}_s$ and $ s_0 = s$ with probability one.

Note that we can still use the performance difference identity \ref{22}
\begin{multline}
V^{{\pi}_{k+1}}_{\tau}(s) - V^{\pi_k}_{\tau}(s) = \frac{1}{1-\gamma} \mathbb{E}_{s'\sim d^{\pi_k}_s} \big[ \langle Q^{\pi_k}_{\tau}(s',\cdot), \pi_{k+1}(\cdot|s') - \pi_k(\cdot|s') \rangle -\tau \H^{{\pi}_{k+1}}(s') + \tau \H^{\pi_k}(s') \big]\\
= \frac{1}{1-\gamma} \mathbb{E}_{s'\sim d^{\pi_k}_s} \big[ 
    \langle Q^{\pi_k,\xi_k}_{\tau}(s',\cdot), {\pi}_{k+1}(\cdot|s') - \pi_k(\cdot|s') \rangle -\tau \H^{{\pi}_{k+1}}(s') + \tau \H^{\pi_k}(s')\\
      - \langle \Delta_k(\cdot,s'), {\pi}_{k+1}(\cdot|s')-\pi_k(\cdot|s') \rangle 
\big]
\end{multline}

By multiplying both sides by -1 and taking expectation with respect to $\xi_{0:k}$ gives  
\begin{align}\label{Valuemean with xi}
&\E_{\xi_{0:k}}\!\left[ V^{\pi_{k}}_{\tau}(s) - V^{{\pi}_{k+1}}_{\tau}(s) \right] \nonumber \\
&\;\;\le \frac{1}{1-\gamma}\,
   \E_{\xi_{0:k}}\E_{s'\sim d^{\pi_k}_s} \Big[
      \langle Q^{\pi_k,\xi_k}_{\tau}(s',\cdot), \pi_{k}(\cdot|s') - {\pi}_{k+1}(\cdot|s') \rangle
      - \tau \H^{\pi_{k}}(s') + \tau \H^{{\pi}_{k+1}}(s') \nonumber \\
&\hspace{5cm}
      + \frac{1}{2\eta_k} \wass_2^2\!\big(\pi_k(\cdot|s'),{\pi}_{k+1}(\cdot|s')\big)
   \Big] + 2\eta_k\sigma_k^2 \nonumber \\
&\;\;\le \E_{\xi_{0:k}}\Big[
      \langle Q^{\pi_k}_{\tau}(s,\cdot), \pi_{k}(\cdot|s) - {\pi}_{k+1}(\cdot|s) \rangle
      - \tau \H^{\pi_{k}}(s) + \tau \H^{{\pi}_{k+1}}(s) \nonumber \\
&\hspace{5cm}
      + \tfrac{1}{2\eta_k}\wass_2^2\!\big({\pi}_{k+1}(\cdot|s),\pi_k(\cdot|s)\big)
   \Big]
   + \tfrac{2\eta_k\sigma_k^2}{1-\gamma}.
\end{align}

Taking expectation with $\xi_{0:k}$ on \ref{eq:w2-lemma4-stochastic} and combine with \ref{Valuemean with xi}, we have:
\begin{align*}
\E_{\xi_{0:k}}\big[\langle Q^{\pi_k,\xi_k}_{\tau}(s,\cdot),\pi_k(\cdot|s)-\pi^\star(\cdot|s)\rangle
+ \tau \H^{\pi_k}(s) - \tau \H^{\pi^\ast}(s)
+ V^{{\pi}_{k+1}}_{\tau}(s) - V^{\pi_k}_{\tau}(s)\big] \\
\le \E_{\xi_{0:k}}\big[\frac{1}{2\eta_k} \wass_2^2(\pi_k(\cdot|s),\pi^\ast(\cdot|s))
- \frac{\lambda\tau}{2} \wass_2^2({\pi}_{k+1}(\cdot|s),\pi^\ast(\cdot|s))\big]
+ \frac{2\eta_k \sigma_k^2}{1-\gamma}.
\end{align*}
 
Finally, averaging over the distribution $s \sim \nu^\ast$ and noting that $s$ and $\xi_{0:k}$ are independent, we have
\begin{equation}\label{1}
\begin{aligned}
&\E_{s\sim\nu^\ast,\xi_{0:k}}\big[\langle Q^{\pi_k,\xi_k}_{\tau}(s,\cdot),\pi^\ast(\cdot|s)-\pi_k(\cdot|s)\rangle - \tau \H^{\pi^\ast}(s)
+ \tau \H^{\pi_k}(s)  +  V^{\pi_{k}}_{\tau}(s) - V^{{\pi}_{k+1}}_{\tau}(s)\big] \\
&\le \E_{s\sim\nu^\ast,\xi_{0:k}}\big[\frac{1}{2\eta_k} \wass_2^2(\pi_k(\cdot|s),\pi^\ast(\cdot|s))
- \frac{\lambda\tau}{2} \wass_2^2({\pi}_{k+1}(\cdot|s),\pi^\ast(\cdot|s))\big]
+ \frac{2\eta_k \sigma_k^2}{1-\gamma}.
\end{aligned}
\end{equation}

Noting that 
\begin{equation}\label{2}
\begin{aligned}
&\E_{\xi_k}\left[ \langle Q^{\pi_k,\xi_k}_{\tau}(s,\cdot), \pi^\ast(\cdot|s) - \pi_k(\cdot|s) \rangle \,\big|\, \xi_{0:k-1} \right]\\
&= \E_{\xi_k} \Big[\langle Q^{\pi_k}_{\tau}(s,\cdot), \pi^\ast(\cdot|s) - \pi_k(\cdot|s)\rangle  
 + \langle \bar{Q}^{\pi_k}_{\tau}(s,\cdot) - Q^{\pi_k}_{\tau}(s,\cdot), \pi^\ast(\cdot|s) - \pi_k(\cdot|s) \rangle \\& + \langle Q^{\pi_k,\xi_k}_{\tau}(s,\cdot) - \bar{Q}^{\pi_k}_{\tau}(s,\cdot), \pi^\ast(\cdot|s) - \pi_k(\cdot|s) \rangle \,\big|\, \xi_{0:k-1} \Big]\\
&\ge \langle Q^{\pi_k}_{\tau}(s,\cdot), \pi^\ast(\cdot|s)-\pi_k(\cdot|s)   \rangle - 2\epsilon_k
\end{aligned}
\end{equation}

The first equality expands $Q^{\pi_k,\xi_k}_{\tau}$ into its expectation 
$Q^{\pi_k}_{\tau}$ plus two error terms, namely the bias 
$\bar{Q}^{\pi_k}_{\tau}-Q^{\pi_k}_{\tau}$ and the stochastic fluctuation 
$Q^{\pi_k,\xi_k}_{\tau}-\bar{Q}^{\pi_k}_{\tau}$. 
Taking conditional expectation w.r.t.\ $\xi_k$ eliminates the mean of the 
fluctuation term. 
Finally, using the uniform error bound 
$\|\bar{Q}^{\pi_k}_{\tau}-Q^{\pi_k}_{\tau}\|_\infty \le \epsilon_k$ 
and noting that both $\pi_k(\cdot|s)$ and $\pi^\ast(\cdot|s)$ are probability measures 
(which implies $\|\pi^\ast(\cdot|s)-\pi_k(\cdot|s)\|_1 \le 2$), 
Hölder's inequality yields the desired bound.

Combining \ref{1} and \ref{2} and using Lemma \ref{VQtransformation},
\begin{equation}\label{101}
\begin{aligned}
\E_{s \sim \nu^\ast, \xi_{0:k}} \Big[ (1-\gamma)\big( V^{\pi_k}_{\tau}(s) & - V^{\pi^\ast}_{\tau}(s) \big) + V^{\bar{\pi}_{k+1}}_{\tau}(s) - V^{\pi_k}_{\tau}(s) \Big] \\
&\le \E_{s \sim \nu^\ast, \xi_{0:k}} \Big[ \frac{1}{2\eta_k} \wass_2^2(\pi_k(\cdot|s),\pi^\ast(\cdot|s))
- \frac{\lambda\tau}{2} \wass_2^2(\bar{\pi}_{k+1}(\cdot|s),\pi^\ast(\cdot|s)) \Big]
+ 2\epsilon_k + \frac{2\eta_k \sigma_k^2}{1-\gamma}.
\end{aligned}
\end{equation}

Decomposing $V^{\pi_{k}}_{\tau}(s) - V^{\pi_{k+1}}_{\tau}(s)$ into $V^{\pi_{k}}_{\tau}(s) - V^{\pi^\ast}_{\tau}(s) - \big( V^{\pi_{k+1}}_{\tau}(s) - V^{\pi^\ast}_{\tau}(s)\big)$, recalling our definition of $J$ \eqref{defination of J} and rearranging the terms in the above inequality, we get
\begin{equation}
\begin{aligned}
&\E_{\xi_{0:k}}[J(\pi^\ast) - J(\pi_{k+1}) +  \lambda\tau \mathcal{D}(\pi_{k+1}, \pi^\star)]\\
&\le \E_{\xi_{0:k-1}}[\gamma(J(\pi^\ast) - J(\pi_k)) + \frac{1}{\eta_k} \mathcal{D}(\pi^\ast, \pi_k)] + \mathcal{O}(\delta + \textcolor{blue}{\tau})+ 2 \epsilon_k + \frac{\eta_k \sigma_k^2}{2(1-\gamma)}.
\end{aligned}
\end{equation}

By choosing $\eta_k=\eta \ge \frac{1}{\gamma\lambda\tau}$ , and for all $k \ge 0$, $\epsilon_k \le \epsilon, \sigma_k \le \sigma$,
we get 

\begin{equation}
\begin{aligned}
\E_{\xi_{0:k-1}}\big[ J(\pi^\ast) - J(\pi_k) + \lambda\tau D(\pi_k,\pi^\ast) \big]
&\le \gamma^k \Big[ J(\pi^\ast) - J(\pi_0) 
+ \lambda\tau D(\pi_0,\pi^\ast) \Big] + \mathcal O(\epsilon+\sigma)
\end{aligned}
\end{equation}

\end{proof}

\newpage

\section{Numerical}  
\label{Numerical}
\subsection{Overall Algorithm}\label{entropy estimation}
\label{Algorithm}
\paragraph{WPPG vs.\ WPPG-I: commonalities and differences}
Both WPPG and WPPG-I are off-policy actor–critic methods built on the same backbone:
(i) replay-based training with 1-step TD targets;
(ii) Double-$Q$ critics with target networks and Polyak averaging;
(iii) multi-sample bootstrap for target construction (average over $K$ next-action samples and take $\min(Q_1,Q_2)$);
(iv) actor updates driven by action-gradient matching, i.e., aligning the policy’s action increment with a noisy target direction, and
(v) \texttt{tanh} squashing that maps actions to box constraints.

\emph{Key difference.} WPPG employs an explicit Tanh–Gaussian policy
$a=\operatorname{Affine}(\tanh(\mu_\theta(s)+\sigma_\theta(s)\odot\varepsilon))$ with a closed-form density (useful if one wishes to incorporate entropy/KL terms).  
WPPG-I uses a latent-conditioned \emph{implicit} policy $a=\operatorname{Affine}(\tanh(f_\theta([s,z])))$, where $z\!\sim\!\mathcal N(0,I)$ is concatenated with the state; the policy distribution is implicit (no closed-form $\log\pi$), and learning relies purely on pathwise gradients through $\nabla_a Q$ and the shared latent variable reforwarding trick.  
Operationally, WPPG controls stochasticity via the Gaussian actor’s output scale, whereas WPPG-I controls it via the \emph{input} latent variables (its scale and dimensionality), enabling richer, state-conditional exploration.

\paragraph{Entropy estimation via plug-in mixture (Gaussian-convolved implicit policy).}
For a given state $s$, let the implicit generator be $g_\theta(s,z)$ with latent prior
$z\sim \nu$ (typically $\nu=\mathcal N(0,I)$). We add a small Gaussian noise
$a = g_\theta(s,z) + \sigma_{\mathrm{ent}}\xi$ with $\xi\sim\mathcal N(0,I)$ so that the induced policy has a density.

Define the Gaussian kernel
\[
\varphi_{\sigma}(x)
=\frac{1}{(2\pi\sigma^2)^{d_a/2}}
\exp\!\left(-\frac{\|x\|_2^2}{2\sigma^2}\right).
\]
Then the (Gaussian-convolved) policy density is the latent-mixture
\[
\pi_\theta(a\mid s)=\mathbb E_{z\sim \nu}\big[\varphi_{\sigma_{\mathrm{ent}}}(a-g_\theta(s,z))\big].
\]

\paragraph{Plug-in approximation.}
Sample $M$ latent codes $\{z_j\}_{j=1}^M\sim \nu$ and form mixture centers
$\mu_j \triangleq g_\theta(s,z_j)$. We approximate the density by
\[
\widehat\pi_{\theta,M}(a\mid s)
\triangleq
\frac{1}{M}\sum_{j=1}^M \varphi_{\sigma_{\mathrm{ent}}}(a-\mu_j).
\]

\paragraph{Entropy estimator.}
Sample $L$ baseline actions from the convolved policy,
\[
\tilde z_\ell\sim \nu,\ \xi_\ell\sim\mathcal N(0,I),\qquad
a_\ell \triangleq g_\theta(s,\tilde z_\ell)+\sigma_{\mathrm{ent}}\xi_\ell,\quad \ell=1,\dots,L,
\]
and estimate the conditional entropy by the plug-in Monte Carlo average
\[
\widehat{\mathcal H}(s)
\triangleq
-\frac{1}{L}\sum_{\ell=1}^L \log \widehat\pi_{\theta,M}(a_\ell\mid s).
\]

\paragraph{Entropy-regularized reward (used in replay).}
Given a transition $(s,a,r,s')$, we store a regularized reward
\[
r^{\mathrm{ent}} \triangleq r + \tau\,\widehat{\mathcal H}(s)
\}
\]

\paragraph{Entropy regularization without explicit log-likelihood.}
Compared with SAC-style maximum-entropy RL, our approach encourages exploration by \emph{increasing policy entropy} using a plug-in mixture estimator that only requires sampling actions, rather than an explicit $\log\pi_\theta(a\mid s)$.
In SAC, the entropy term is computed from a \emph{single} action sample via $-\log\pi_\theta(a\mid s)$, which necessitates a tractable density (and often its gradients/``score'' $\nabla_a\log\pi_\theta$ through reparameterization), making it unsuitable for implicit policies.

In contrast, we estimate the entropy at each state by drawing multiple latent samples $\{\mu_j=g_\theta(s,z_j)\}_{j=1}^M$ to form a Gaussian-kernel mixture
$\widehat\pi_{\theta,M}(a\mid s)=\frac1M\sum_{j=1}^M\varphi_{\sigma}(a-\mu_j)$,
and then average $-\log \widehat\pi_{\theta,M}(a_\ell\mid s)$ over $L$ baseline actions $\{a_\ell\}_{\ell=1}^L$.
This ``many-sample'' plug-in estimator reduces the stochasticity of the entropy signal: instead of relying on a single-point log-density as in SAC, we aggregate information from $M\times L$ samples per state, yielding a substantially lower-variance regularization term while remaining fully compatible with implicit generators.
Crucially, the procedure does not require a closed-form $\log\pi_\theta$ nor access to the score $\nabla_a\log\pi_\theta$; the only requirement is the ability to sample $g_\theta(s,z)$ (plus a small Gaussian smoothing noise), which makes the method broadly applicable to expressive implicit policies.

\begin{algorithm}[H]
\caption{Plug-in Mixture Entropy Estimation for Gaussian-Convolved Implicit Policy}
\label{alg:entropy_plugin}
\begin{algorithmic}
\Require State batch $\{s_i\}_{i=1}^B$; generator $g_\theta(s,z)$; noise scale $\sigma$; 
        mixture samples $M$; baseline samples $L$; latent prior $\nu$ (typically $\mathcal N(0,I)$).
\Ensure Entropy estimates $\{\widehat{\mathcal H}(s_i)\}_{i=1}^B$
\State Define Gaussian kernel density in $d_a$ dimensions:
\[
\varphi_\sigma(x)=\frac{1}{(2\pi\sigma^2)^{d_a/2}}\exp\!\left(-\frac{\|x\|^2}{2\sigma^2}\right)
\]
\For{$i=1$ to $B$}
  \State \textbf{Mixture centers (MC for (7)):} draw $z_{i,1:M}\sim \nu$, set $\mu_{i,j}\leftarrow g_\theta(s_i,z_{i,j})$ for $j=1..M$
  \State \textbf{Baseline actions (for (11)):} draw $\tilde z_{i,1:L}\sim \nu$, $\xi_{i,1:L}\sim\mathcal N(0,I)$
  \State \hspace{1.6em} set $a_{i,\ell}\leftarrow g_\theta(s_i,\tilde z_{i,\ell})+\sigma\,\xi_{i,\ell}$ for $\ell=1..L$
  \State \textbf{Plug-in density (9):} for each $\ell$, compute
  \[
  \widehat\pi_{\theta,M}(a_{i,\ell}\mid s_i)\leftarrow \frac{1}{M}\sum_{j=1}^M \varphi_\sigma(a_{i,\ell}-\mu_{i,j})
  \]
  \State \textbf{Entropy estimate (11):}
  \[
  \widehat{\mathcal H}(s_i)\leftarrow -\frac{1}{L}\sum_{\ell=1}^L \log \widehat\pi_{\theta,M}(a_{i,\ell}\mid s_i)
  \]
\EndFor
\State \Return $\{\widehat{\mathcal H}(s_i)\}_{i=1}^B$
\end{algorithmic}
\end{algorithm}

\begin{algorithm}[H]
\caption{WPPG with Replay and Double-$Q$ Critics (Gaussian policy) + On-the-fly Plug-in Entropy Regularization}
\label{alg:wpg_ent_online_tau}
\begin{algorithmic}
\Require Initialize actor $\pi_\theta(a|s)=\mathcal{N}(\mu_\theta(s),\Sigma_\theta(s))$ with Tanh squash; 
         twin critics $Q_{w_1},Q_{w_2}$ and targets $\bar Q_{w_1},\bar Q_{w_2}$; target actor $\bar\pi_\theta$; 
         replay buffer $\mathcal{D}$; step size $\eta$, noise scale $\tau$, samples per state $K$, discount $\gamma$, Polyak $\sigma$.
\Require \textbf{Entropy params:} convolution noise std $\sigma_{\mathrm{ent}}$; mixture size $M$; baseline size $L$.
\For{each episode}
  \State Initialize $s_0$
  \For{$t=0$ to $T-1$}
    \State Sample $a_t \sim \pi_\theta(\cdot|s_t)$ and \textbf{Gaussian-convolve}:
           $a_t \leftarrow a_t + \sigma_{\mathrm{ent}}\xi_t,\;\xi_t\sim\mathcal N(0,I)$
    \State Execute $a_t$, observe $(r_t,s_{t+1},d_t)$
    \Statex
    \State \textbf{On-the-fly entropy estimation at $s_t$:}
    \State Compute $\widehat{\mathcal H}(s_t)$ by calling Algorithm~\ref{alg:entropy_plugin} 
    \State Form entropy-regularized reward:
           $r_t^{\mathrm{ent}} \leftarrow r_t + \tau\,\widehat{\mathcal H}(s_t)$
           
    \State Store $(s_t,a_t,r_t^{\mathrm{ent}},s_{t+1},d_t)$ into $\mathcal{D}$
    \If{len$(\mathcal{D}) \ge$ batch\_size}
      \State Sample minibatch $\{(s_i,a_i,r_i^{\mathrm{ent}},s'_i,d_i)\}_{i=1}^B$ from $\mathcal{D}$
      \Statex
      \State \textbf{Compute 1-step TD targets (multi-sample bootstrap using target nets):}
      \State For each $s'_i$, draw $\epsilon_{i,1:K}\sim\mathcal{N}(0,I)$ and set $a'_{i,k} \leftarrow \bar\pi_\theta(s'_i;\epsilon_{i,k})$
      \State $\hat Q_i \leftarrow \frac{1}{K}\sum_{k=1}^K \min\!\big(\bar Q_{w_1}(s'_i,a'_{i,k}), \bar Q_{w_2}(s'_i,a'_{i,k})\big)$
      \State $y_i \leftarrow r_i^{\mathrm{ent}} + \gamma(1-d_i)\,\hat Q_i$
      \Statex
      \State \textbf{Critic update (train both critics):}
      \State $w_j \leftarrow w_j - \beta_Q \nabla_{w_j}\frac{1}{B}\sum_i \big(Q_{w_j}(s_i,a_i)-y_i\big)^2,\quad j\in\{1,2\}$
      \Statex
      \State \textbf{Actor update (WPPG step with action-sample direction):}
      \State For each $s_i$, draw shared $\epsilon_{i,1:K}$ and form $a_{i,k}=\pi_\theta(s_i;\epsilon_{i,k})$; let $A_i=[a_{i,1:K}]$
      \State Compute $q_{i,k}=\min\!\big(Q_{w_1}(s_i,a_{i,k}),Q_{w_2}(s_i,a_{i,k})\big)$
      \State Obtain $\nabla_a Q$ at samples: $G_i=\big[\nabla_a q_{i,k}\big]_{k=1}^K$
      \State Form noisy target direction: $\Delta^\star_i \leftarrow \eta\,G_i + \xi_i$, where $\xi_i\sim \mathcal{N}\!\left(0,\,2\tau\eta\,I\right)$
      \State Re-sample $A'_i=\pi_\theta(s_i;\epsilon_{i,1:K})$ with the \emph{same} $\epsilon$ and define $\Delta_i \leftarrow A'_i - A_i$
      \State Update actor by matching directions:
             $\theta \leftarrow \theta - \beta_\pi \nabla_\theta \frac{1}{BK}\sum_{i,k} \|\Delta_{i,k}-\Delta^\star_{i,k}\|_2^2$
      \Statex
      \State \textbf{Target updates (Polyak):} 
             $\bar w_j \leftarrow \sigma w_j + (1-\sigma)\bar w_j,\;
              \bar\theta \leftarrow \sigma \theta + (1-\sigma)\bar\theta,\; j\in\{1,2\}$
    \EndIf
    \State $s_{t+1}\leftarrow s'$
  \EndFor
\EndFor
\end{algorithmic}
\end{algorithm}

\begin{algorithm}[H]
\caption{WPPG-I with Replay and Double-$Q$ Critics (Implicit Policy) + On-the-fly Plug-in Entropy Regularization}
\label{alg:wpgi_ent_online_tau}
\begin{algorithmic}
\Require Implicit actor $a=g_\theta(s,z)$ with Tanh squash ($z\!\sim\!\mathcal N(0,I_M)$); 
         twin critics $Q_{w_1},Q_{w_2}$ and targets $\bar Q_{w_1},\bar Q_{w_2}$; target actor $\bar g_\theta$; 
         replay buffer $\mathcal D$; step size $\eta$, noise scale $\tau$, samples per state $K$, discount $\gamma$, Polyak $\sigma$.
\Require \textbf{Entropy params:} convolution noise std $\sigma_{\mathrm{ent}}$; mixture size $M$; baseline size $L$.
\For{each episode}
  \State Initialize $s_0$
  \For{$t=0$ to $T-1$}
    \State Sample $z_t\!\sim\!\mathcal N(0,I_M)$, set $a_t\!=\!g_\theta(s_t,z_t)$ and \textbf{Gaussian-convolve}:
           $a_t \leftarrow a_t + \sigma_{\mathrm{ent}}\xi_t,\;\xi_t\sim\mathcal N(0,I)$
    \State Step env, observe $(r_t,s_{t+1},d_t)$
    \Statex
    \State \textbf{On-the-fly entropy estimation at $s_t$:}
    \State Compute $\widehat{\mathcal H}(s_t)$ by calling Algorithm~\ref{alg:entropy_plugin} with generator $g_\theta$ and noise $\sigma_{\mathrm{ent}}$
    \State Form entropy-regularized reward:
           $r_t^{\mathrm{ent}} \leftarrow r_t + \tau\,\widehat{\mathcal H}(s_t)$
    \State Store $(s_t,a_t,r_t^{\mathrm{ent}},s_{t+1},d_t)$ into $\mathcal D$
    \If{len$(\mathcal D)\!\ge\!$ batch\_size}
      \State Sample minibatch $\{(s_i,a_i,r_i^{\mathrm{ent}},s'_i,d_i)\}_{i=1}^B$ from $\mathcal D$
      \Statex
      \State \textbf{Compute 1-step TD targets (multi-sample bootstrap):}
      \State For each $s'_i$, draw $z'_{i,1:K}\!\sim\!\mathcal N(0,I_M)$ and set $a'_{i,k}\!\leftarrow\!\bar g_\theta(s'_i,z'_{i,k})$
      \State $\hat Q_i \leftarrow \frac{1}{K}\sum_{k=1}^K \min\!\big(\bar Q_{w_1}(s'_i,a'_{i,k}),\,\bar Q_{w_2}(s'_i,a'_{i,k})\big)$
      \State $y_i \leftarrow r_i^{\mathrm{ent}} + \gamma(1-d_i)\,\hat Q_i$
      \Statex
      \State \textbf{Critic update (train both critics):}
      \State $w_j \leftarrow w_j - \beta_Q \nabla_{w_j}\frac{1}{B}\sum_i \big(Q_{w_j}(s_i,a_i)-y_i\big)^2,\quad j\in\{1,2\}$
      \Statex
      \State \textbf{Actor update (direction matching with shared noise):}
      \State For each $s_i$, draw shared $z_{i,1:K}\!\sim\!\mathcal N(0,I_M)$ and set $a^{(0)}_{i,k}\!\leftarrow\!g_\theta(s_i,z_{i,k})$
      \State Compute $q_{i,k}=\min\!\big(Q_{w_1}(s_i,a^{(0)}_{i,k}),Q_{w_2}(s_i,a^{(0)}_{i,k})\big)$
      \State Obtain $G_i=\big[\nabla_a q_{i,k}\big]_{k=1}^K$ at $a^{(0)}$ (stop grad to critics)
      \State Form target direction: $\Delta^\star_i \leftarrow \eta\,G_i + \xi_i$, where $\xi_i\!\sim\!\mathcal N\!\left(0,\,2\tau\eta\,I\right)$
      \State Reforward with the \emph{same} $z$: $a^{(1)}_{i,k}\!\leftarrow\!g_\theta(s_i,z_{i,k})$, define $\Delta_{i,k}\!\leftarrow\!a^{(1)}_{i,k}-a^{(0)}_{i,k}$
      \State Update actor:
             $\theta \leftarrow \theta - \beta_\pi \nabla_\theta \frac{1}{BK}\sum_{i,k}\|\Delta_{i,k}-\Delta^\star_{i,k}\|_2^2$
      \Statex
      \State \textbf{Target updates (Polyak):} 
             $\bar w_j \leftarrow \sigma w_j + (1-\sigma)\bar w_j,\;
              \bar\theta \leftarrow \sigma \theta + (1-\sigma)\bar\theta,\; j\in\{1,2\}$
    \EndIf
    \State $s_{t+1}\leftarrow s'$
  \EndFor
\EndFor
\end{algorithmic}
\end{algorithm}

\subsection{Implementations}
\label{Implementation}
\subsubsection{Actor and Policy}
All our models and baselines are implemented under the standard actor–critic framework. 
Below we provide the implementation details for the actor and critic components separately.
\paragraph{Action Squashing.}
For consistency across methods, we apply a \emph{tanh} squashing function to map sampled actions into the valid box $[a_{\min},a_{\max}]$ for all algorithms. This squashing is crucial for the implicit policy: without it, when the injected-latent dimension is high, many actions are hard-clipped at the bounds, which prevents meaningful exploration and gradients, often leading to training failure. Below we present concise formulations of the two actors used.

\paragraph{Tanh–Gaussian MLP Policy (used in WPPG/SAC/WPO).}
Given state $s\in\mathbb{R}^{S}$, the actor outputs $\mu_\theta(s),\log\sigma_\theta(s)\in\mathbb{R}^{A}$ and samples
\[
a \;=\; \tfrac{a_{\max}-a_{\min}}{2}\odot\tanh\!\big(\mu_\theta(s)+\sigma_\theta(s)\odot\varepsilon\big)
\;+\; \tfrac{a_{\max}+a_{\min}}{2}, 
\qquad \varepsilon\sim\mathcal N(0,I_A),
\]
i.e., a tanh-squashed Gaussian mapped to $[a_{\min},a_{\max}]$ via an MLP producing $(\mu_\theta,\log\sigma_\theta)$.

\paragraph{Noise-Conditioned Deterministic Policy (used in WPPG-I).}
Given state $s\in\mathbb{R}^{S}$ and latent variables $z\in\mathbb{R}^{M}$,
\[
a \;=\; \tfrac{a_{\max}-a_{\min}}{2}\odot\tanh\!\big(f_\theta([s,z])\big)
\;+\; \tfrac{a_{\max}+a_{\min}}{2},
\qquad z\sim\mathcal N(0,I_M),
\]
where $f_\theta$ is an MLP taking the concatenated input $[s,z]$. This defines an implicit policy (no closed-form density) with tanh-squashed outputs mapped to $[a_{\min},a_{\max}]$.

\subsubsection{Critic}
\paragraph{Critic Learning Target}
For all off-policy algorithms (SAC, WPPG, WPPG-I, WPO), the critic is trained with 1-step TD targets that average over $K$ bootstrap action samples and use Double-$Q$ when available:
\[
y_t \;=\; r_t \,+\, \gamma (1-d_t)\; \frac{1}{K}\sum_{k=1}^{K}
\min_{j\in\{1,2\}}\, Q_{\bar w_j}\!\big(s_{t+1},\, a'_{t+1,k}\big), 
\qquad a'_{t+1,k} \;=\; g_{\bar\theta}\!\big(s_{t+1}, \varepsilon_k\big),\;\; \varepsilon_k\sim\mathcal N(0,I).
\]
Here, \(j\in\{1,2\}\) indexes the two target critics used by Double-$Q$ (the per-sample minimum is taken), and the outer average is over \(K\) target actions drawn from the target actor \(g_{\bar\theta}\).
For single-$Q$ methods (e.g., WPO \emph{uses a single $Q$}, or the single-$Q$ WPPG ablation), replace \(\min_{j\in\{1,2\}} Q_{\bar w_j}\) by \(Q_{\bar w}\).
In contrast, PPO retains its on-policy generalized advantage estimation (GAE) for actor updates.

\subsection{Hyperparameters}

\paragraph{Neural Network Architecture}
All experiments are based on two standard network configurations: a larger network with hidden sizes (256, 256) and ReLU activation, and a smaller network with hidden sizes (64, 64) and Tanh activation. We use the larger network for Hopper, Humanoid, and HalfCheetah, and the smaller network for all other tasks. The same choice is applied uniformly across all models, and the actor and critic share the same network architecture.
\paragraph{Replay Buffer}
For consistency, all off-policy algorithms (SAC, WPPG, WPPG-I, WPO) use the same replay buffer configuration as summarized in Table~\ref{tab:replay_buffer}, ensuring identical storage capacity, sampling scheme, and update frequency across methods.
\paragraph{Training Setup}
All off-policy models share the same basic training setup: each is trained for $1\times 10^6$ timesteps, with evaluation performed every $2000$ steps. Target networks are updated via Polyak averaging to stabilize critic training. Both actor and critic use a learning rate of $3\times 10^{-4}$. The discount factor is set to $\gamma=0.99$ for all tasks, except Swimmer where $\gamma=0.9999$. The configuration is summarized in Table~\ref{tab:Training_Setup}. For PPO, hyperparameters are environment-specific and detailed below.

\begin{table}[ht!]
\centering
\begin{minipage}{0.48\linewidth}
\centering
\caption{Replay buffer.}
\label{tab:replay_buffer}
\begin{tabular}{l c}
\toprule
\textbf{Hyperparameter} & \textbf{Value} \\
\midrule
Buffer size              & $1{,}000{,}000$ \\
Batch size               & $256$ \\
Learning starts          & $10{,}000$ \\
Train frequency          & $1$ (step) \\
Gradient steps per update& $1$ \\
Number of environments   & $1$ \\
\bottomrule
\end{tabular}
\end{minipage}\hfill
\begin{minipage}{0.48\linewidth}
\centering
\caption{Training Setup.}
\label{tab:Training_Setup}
\begin{tabular}{l c}
\toprule
\textbf{Hyperparameter} & \textbf{Value} \\
\midrule
Discount factor $\gamma$          & $0.99$ \\
Polyak coefficient                & $0.005$ \\
Learning rate (actor/critic)      & $3\times 10^{-4}$ \\
Target update interval            & $1$ \\
Total timesteps                   & $1{,}000{,}000$ \\
Optimizer& Adam \\
\bottomrule
\end{tabular}
\end{minipage}
\end{table}

\paragraph{Model Specific Hyperparameters}
The hyperparameters of all models are summarized in the tables below. 
For WPO and SAC, we follow the settings reported in the WPO paper, 
while the hyperparameters of PPO are taken from RL Zoo. 
Although we find that tuning hyperparameters for each environment can often improve performance, 
for simplicity and fairness of comparison we adopt a single unified set of hyperparameters 
for all off-policy methods across all tasks, which yields stable and competitive learning performance. 

\begin{table*}[ht!]
\centering
\begin{minipage}{0.48\linewidth}
\centering
\caption{WPPG-I hyperparameters.}
\label{tab:wpgi_hparams}
\begin{tabular}{l c}
\toprule
\textbf{Hyperparameter} & \textbf{Default Value} \\
\midrule
Action samples            & $32$ \\
Step size $\eta$          & $0.1$ \\
Entropy scale $\tau$        & $0.0001$ \\
Actor Latent Dimension     & $\tfrac{1}{3}\times$ State Dimension \\
\bottomrule
\end{tabular}
\end{minipage}\hfill
\begin{minipage}{0.48\linewidth}
\centering
\caption{WPPG hyperparameters.}
\label{tab:wpg_hparams}
\begin{tabular}{l c}
\toprule
\textbf{Hyperparameter} & \textbf{Default Value} \\
\midrule
Action samples            & $32$ \\
Step size $\eta$          & $0.1$ \\
Entropy scale $\tau$        & $0.0001$ \\
\bottomrule
\end{tabular}
\end{minipage}

\vspace{0.5cm}

\begin{minipage}{0.48\linewidth}
\centering
\caption{SAC hyperparameters.}
\label{tab:sac_hparams}
\begin{tabular}{l c}
\toprule
\textbf{Hyperparameter} & \textbf{Default Value} \\
\midrule
Entropy Coefficient $\alpha$ & $0.001$ \\
Maximum Policy Variance      & $\exp(4)$ \\
Minimum Policy Variance      & $\exp(-10)$ \\
\bottomrule
\end{tabular}
\end{minipage}\hfill
\begin{minipage}{0.48\linewidth}
\centering
\caption{WPO hyperparameters.}
\label{tab:wpo_hparams}
\begin{tabular}{l c}
\toprule
\textbf{Hyperparameter} & \textbf{Default Value} \\
\midrule
KL Mean Penalty $\alpha_{\mu}$     & $0.001$ \\
KL Variance Penalty $\alpha_{\Sigma}$ & $0.001$ \\
Action samples                     & $32$ \\
\bottomrule
\end{tabular}
\end{minipage}
\end{table*}

\subsection{Additional Experiment Results}
\label{Additional Results}

\paragraph{Multi-Run Evaluation}
To more comprehensively demonstrate the behavior of WPPG and WPPG-I, we further evaluate both methods with multiple training runs. Specifically, each algorithm is trained 5 times with different random seeds. In the corresponding plots, the solid line denotes the mean return across the 5 runs, while the shaded area indicates the range between the minimum and maximum returns over these runs (see Fig.\ref{fig:Multi-Run Evaluation}).

\begin{figure}
    \centering
    \includegraphics[width=0.9\linewidth]{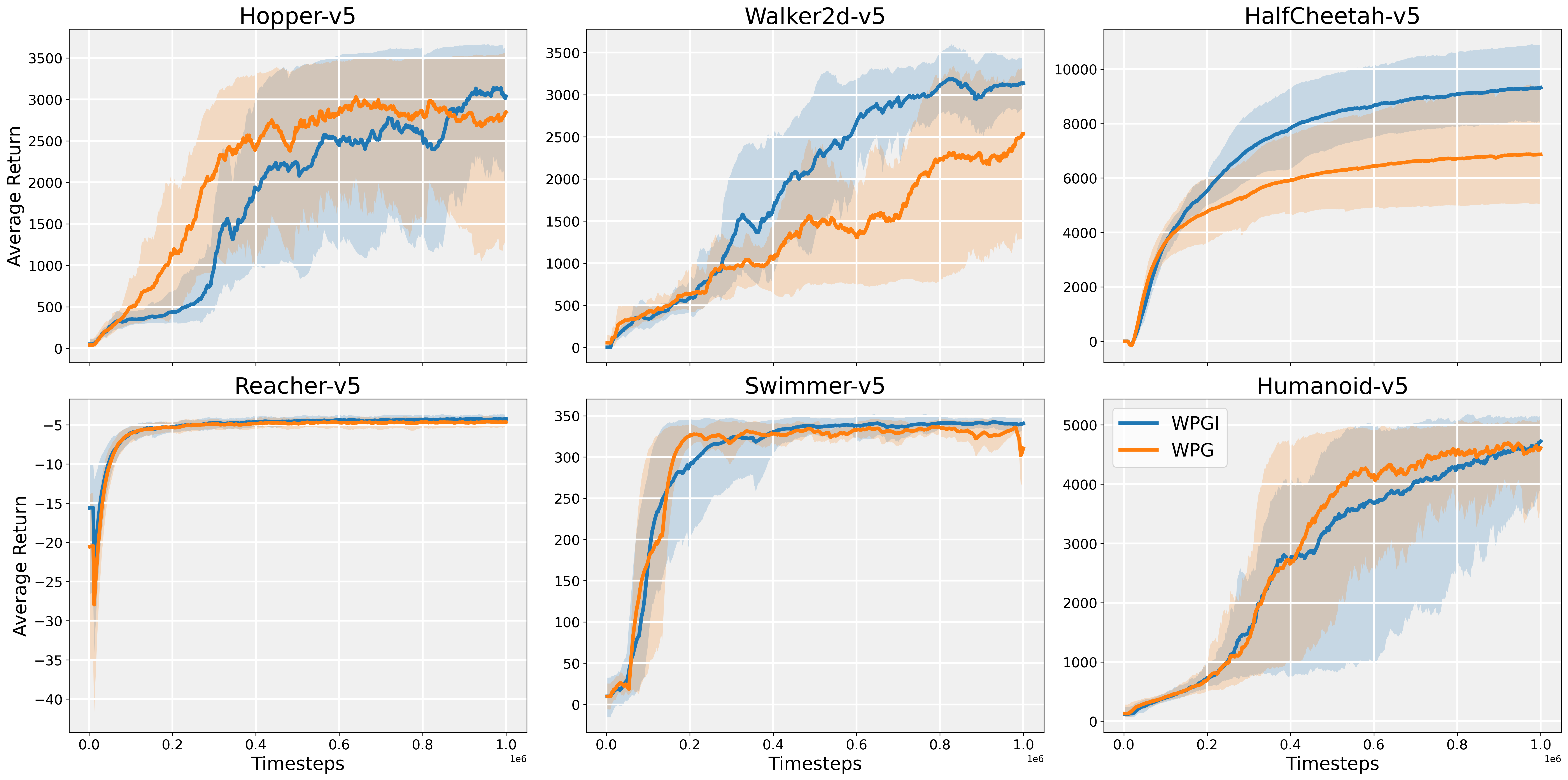}
    \caption{Multi-Run Evaluation}
    \label{fig:Multi-Run Evaluation}
\end{figure}

\paragraph{Combined Humanoid Task}
Our method shares some similarities with SAC in that both are based on entropy regularization and use the action gradient of the Q-function for policy updates. However, the key advantage of WPPG lies in its ability to train an implicit policy. To better showcase this benefit, we follow the construction in WPO \cite{pfau_wasserstein_2025} and create a combined task that increases the action dimensionality: multiple Humanoid environments are run in parallel, their states are concatenated and fed into a single agent, which outputs the concatenated actions jointly. As shown in the combined Humanoid task, WPPG-I converges to consistently higher returns than SAC, indicating that the implicit policy is able to discover action distributions that achieve higher rewards.
 (see Fig.\ref{fig:combined task}).

\begin{figure}
    \centering
    \includegraphics[width=0.8\linewidth]{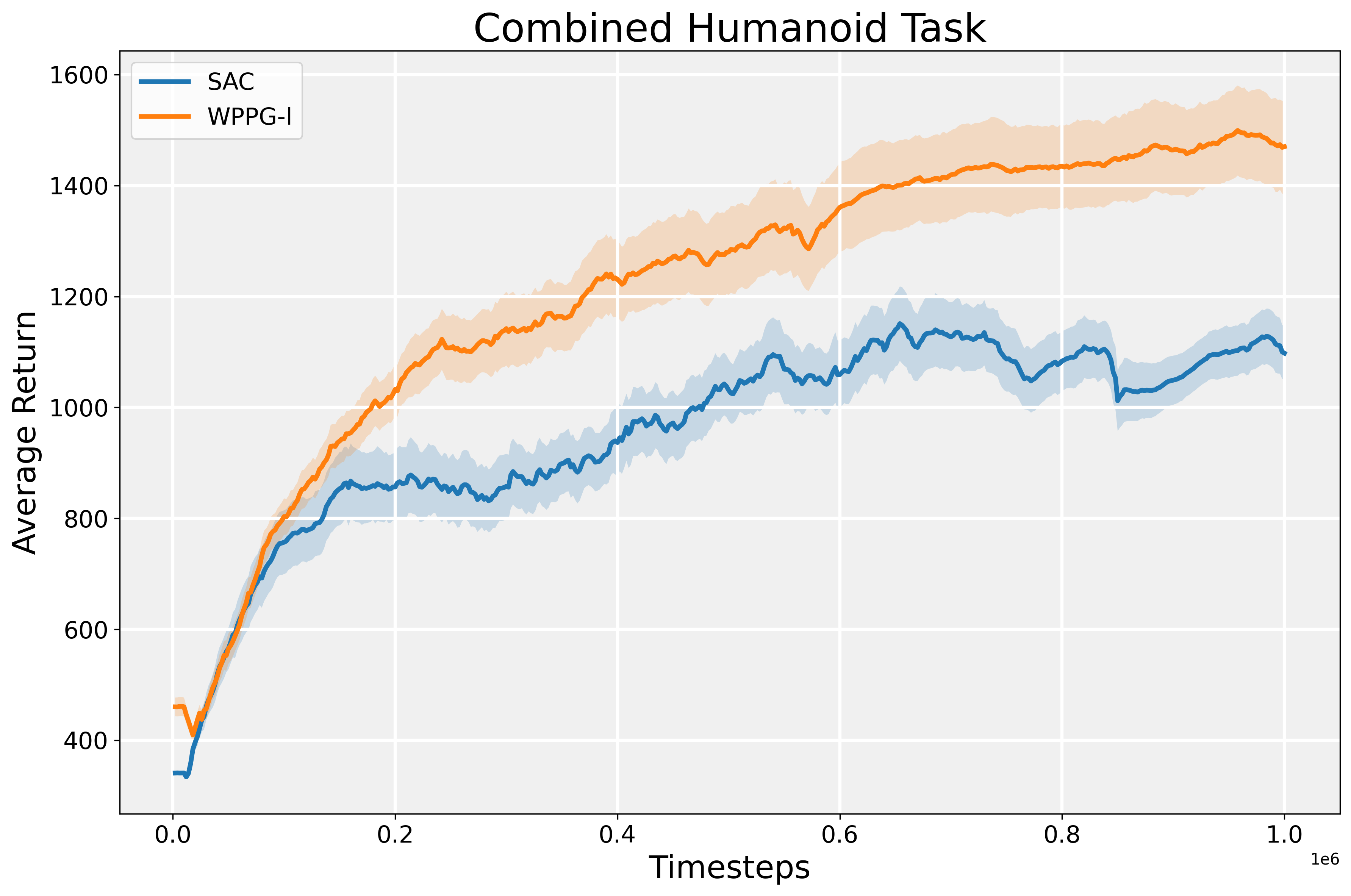}
    \caption{Combined Humanoid Task}
    \label{fig:combined task}
\end{figure}

\paragraph{Ablation on $\tau$.} 
In the preceding analysis, we showed that the parameter $\tau$ originates from entropy regularization of the policy. 
Unlike SAC and related methods that explicitly add an entropy penalty term into the $Q$-function fitting objective, WPPG does not require such a penalty. 
Instead, Gaussian noise is injected when computing the movement direction of action samples, where the scale of the injected noise $\tau$ corresponds to the magnitude of the entropy penalty. 

To study the impact of $\tau$, we conducted ablation experiments on Humanoid environment.
As illustrated in Figure~\ref{fig:ablation-combined}, on Humanoid we observe that injecting noise with $\tau$ in the range $[0, 0.01]$ significantly accelerates convergence, while larger values $0.1$ slows it down. 
This reflects a clear exploration–exploitation trade-off: noise injection encourages the policy to maintain entropy, thereby enabling exploration of richer reward information, but excessive noise hampers the ability of $\nabla_a Q(s,a)$ to provide useful guidance for policy updates. 

\begin{figure}[htbp]
    \centering
    \begin{minipage}[b]{0.35\linewidth}
        \centering
        \includegraphics[width=\linewidth]{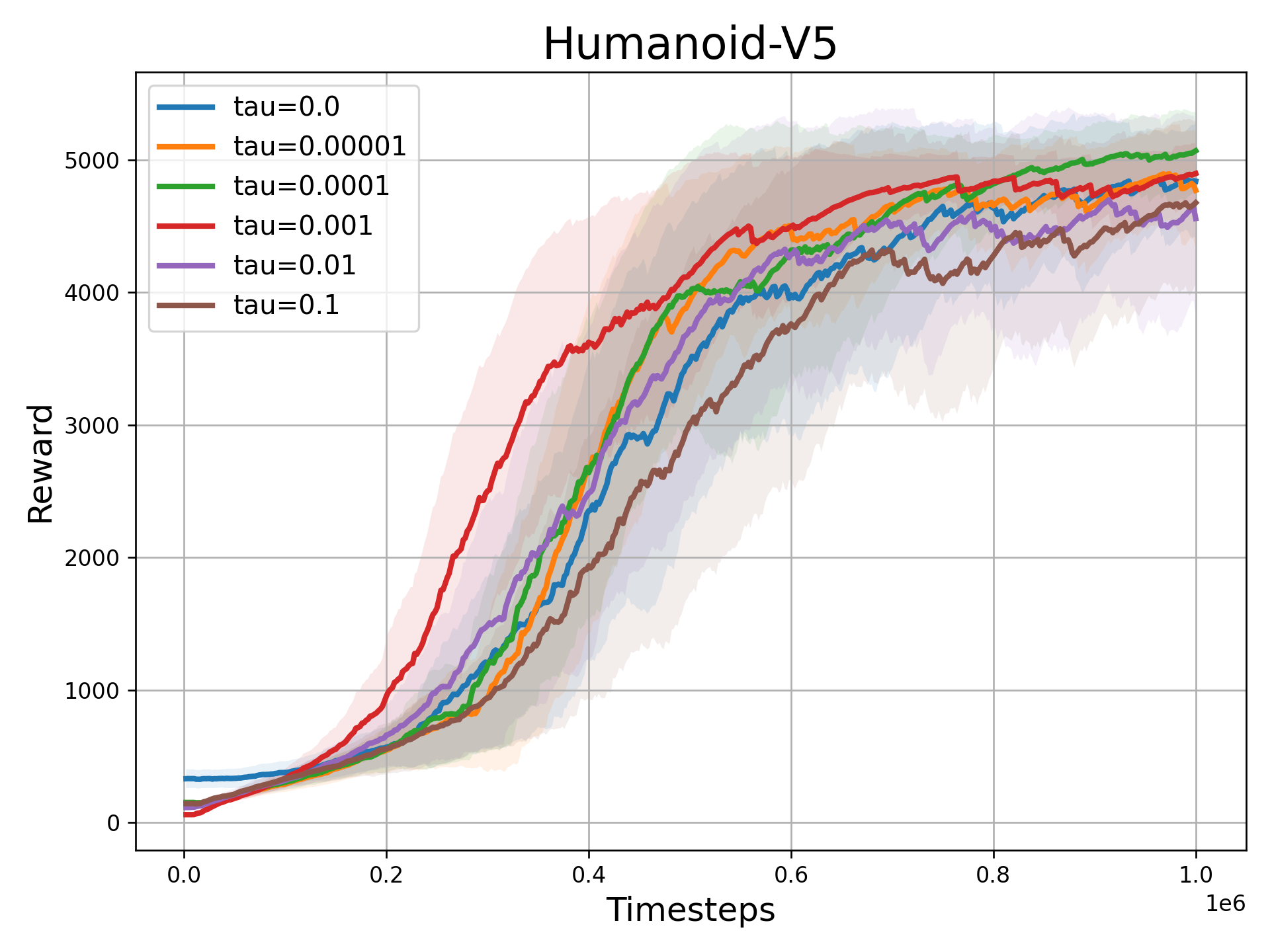}
        \label{fig:ablation-humanoid}
    \end{minipage}
    \qquad
    \begin{minipage}[b]{0.35\linewidth}
        \centering
        \includegraphics[width=\linewidth]{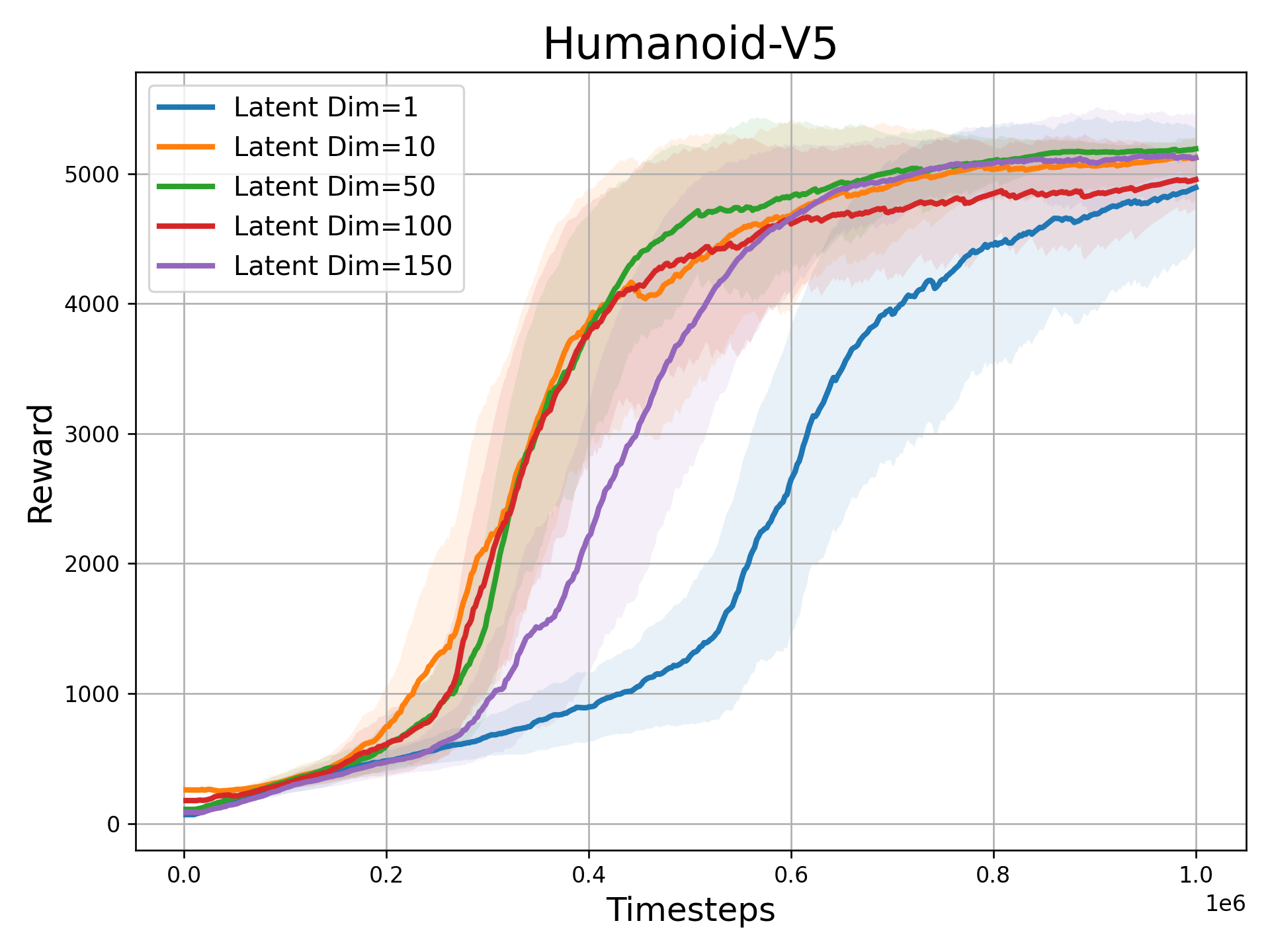}
        \label{fig:ablation-hopper}
    \end{minipage}
    \caption{Ablation study on $\tau$ (left) and $Latent\ Dimension$ (right).}
    \label{fig:ablation-combined}
\end{figure}
\paragraph{Ablation on Latent Dimension} 
We further evaluate the effect of the dimension of latent variable $z$ in our implicit generative model on WPPG-I in the Humanoid environment. 
When the latent dimension is as small as 1, the model learns slowly due to insufficient stochasticity. 
With moderate dimensions (e.g., 10, 50, 100), learning is significantly accelerated, indicating that 
a reasonable amount of latent variables enhances exploration without overwhelming the policy. 
However, when the latent dimension becomes too large (e.g., 150), excessive non-informative variables begin to 
dominate the input and degrade learning speed. 
Empirically, we find that setting the latent dimension about one-third of the state 
dimension provides a good balance between exploration and stability.

\paragraph{Ablation on Double Q Trick} 
We also evaluated the single-$Q$ variant of WPPG across all environments, and found that it outperforms WPO on nearly every task, with the corresponding results provided in the Appendix \ref{Additional Results}. In addition, consistent with prior findings, adopting double-$Q$ further improves WPPG by both stabilizing training and enhancing overall performance.

\paragraph{Ablation on Double Q Function}
Double-$Q$ plays a crucial role for WPPG. As shown in the figures, although the single-$Q$ variant of WPPG outperforms WPO on most environments, it fails to achieve fast and stable learning on challenging tasks such as Humanoid. Beyond stability, the use of double-$Q$ also opens up interesting directions for further exploration; for example, one could choose the $Q$-function with the smaller gradient magnitude to provide the action-sample update direction. We leave such extensions for future work.

\begin{figure}[ht!]
    \centering
    \includegraphics[width=0.9\linewidth]{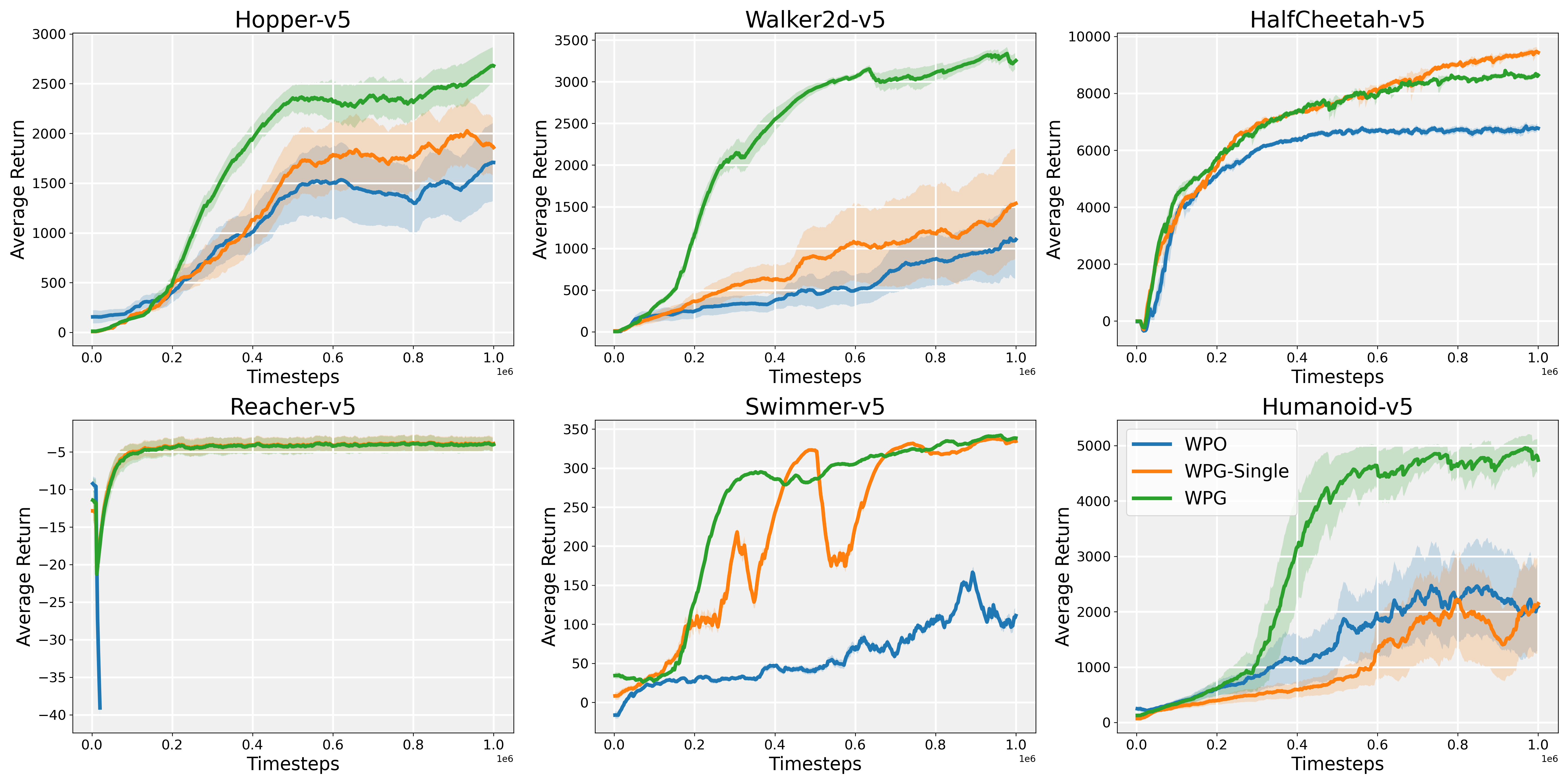}
    \caption{Ablation on Double Q Function}
    \label{fig:ablation-DQ}
\end{figure}

\paragraph{Additional Ablation Study}
Beyond the main results, we also conduct additional ablation studies on Humanoid-v5 and HalfCheetah-v5, systematically varying key hyperparameters (e.g., the Wasserstein step size 
$\eta$, the number of sampled actions, and the latent dimension of the implicit policy). These experiments, reported in the supplementary material, further validate the robustness of our method and illustrate how performance and stability depend on these design choices.
 (see Fig.\ref{fig:half-ablation} and Fig.\ref{fig:human-ablation}).

\begin{figure}
    \centering
    \includegraphics[width=0.9\linewidth]{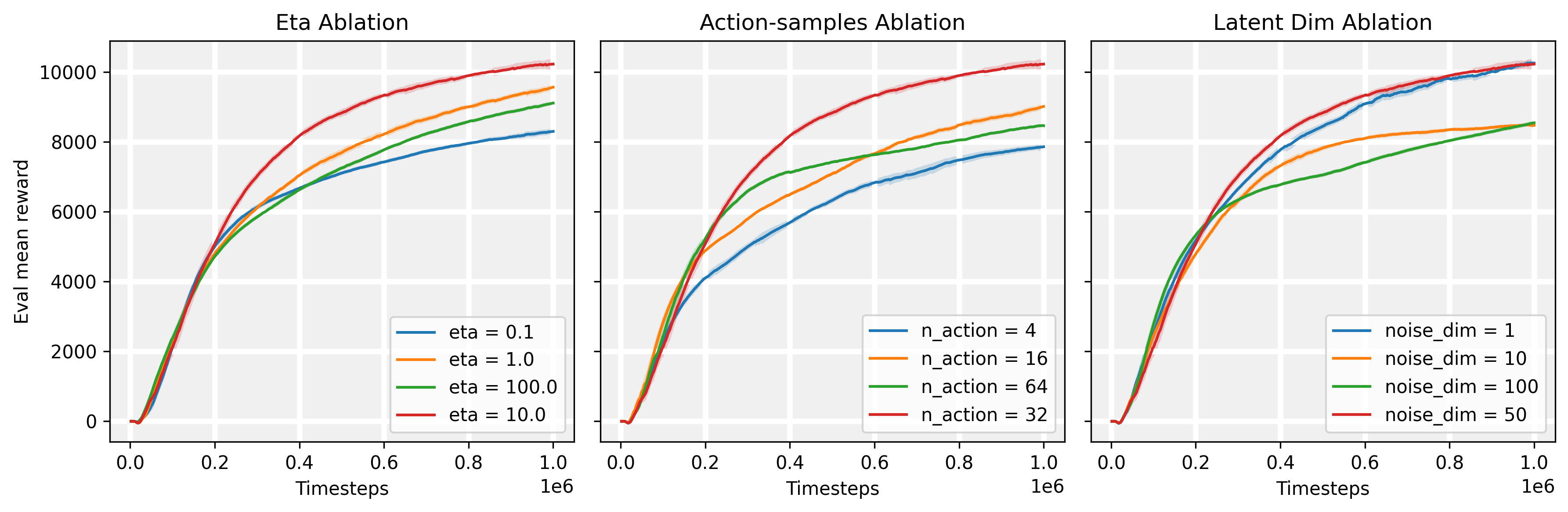}
    \caption{Additional Ablation on Eta, Action Samples and Latent Dim with HalfCheetah-v5 Task}
    \label{fig:half-ablation}
\end{figure}

\begin{figure}
    \centering
    \includegraphics[width=0.9\linewidth]{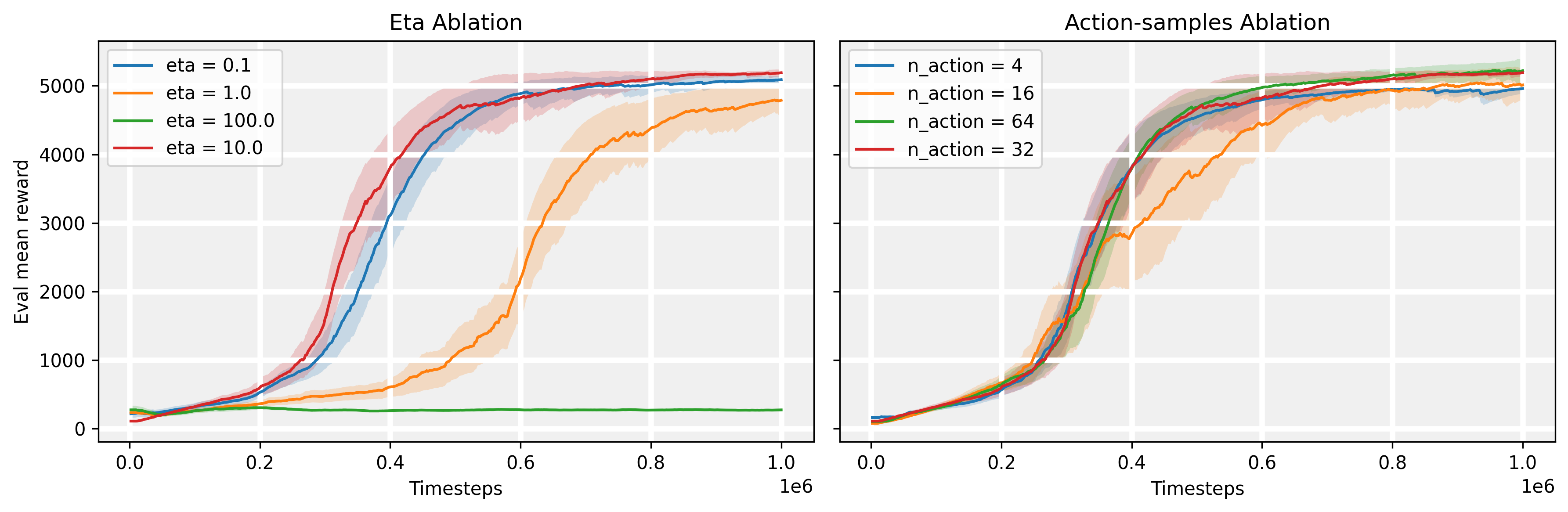}
    \caption{Additional Ablation on Eta, Action Samples and Latent Dim with Humanoid-v5 Task}
    \label{fig:human-ablation}
\end{figure}

\end{document}